%% file: main.tex
\documentclass[journal]{IEEEtran}
\IEEEoverridecommandlockouts
\usepackage{cite}
\usepackage{amsmath,amssymb,amsfonts,amsthm}
\usepackage{mathtools}
\usepackage[font=small,labelfont=bf]{caption}
\usepackage[font=small,labelfont=bf]{subcaption}
\usepackage{graphicx}
\usepackage{textcomp}
\usepackage{microtype}
\usepackage{dsfont}
\usepackage[english]{babel}
\usepackage[autostyle]{csquotes}
\usepackage{hyperref}
\usepackage{cleveref}
\usepackage[usenames,dvipsnames]{color}
\usepackage{afterpage}
\usepackage{cases}
\usepackage{bm}
\usepackage{algorithm}
\usepackage{algpseudocodex}
\usepackage{listings}
\usepackage{xparse}
\usepackage{nicefrac}
\usepackage{microtype}

\graphicspath{ {./images} }

\def\BibTeX{{\rm B\kern-.05em{\sc i\kern-.025em b}\kern-.08em
    T\kern-.1667em\lower.7ex\hbox{E}\kern-.125emX}}

\newtheorem{theorem}{Theorem}
\newtheorem{lemma}[theorem]{Lemma}
\newtheorem*{remark}{Remark}

\theoremstyle{definition}

\begin{document}

\title{On The Evaluation of\\Collision Probability along a Path
}
\author{\IEEEauthorblockN{Lorenzo Paiola$^{1,2}$}, %
    \IEEEauthorblockN{Giorgio Grioli$^{1,2}$}, %
    \IEEEauthorblockN{Antonio Bicchi$^{1,2}$}
    \thanks{$^{1}$Soft Robotics for Human Cooperation and Rehabilitation, Istituto Italiano di Tecnologia, Genova 16163, Italy
        ({\tt\small lorenzo.paiola@iit.it}).}
    \thanks{$^{2}$Department of Information Engineering and Centro di Ricerca \enquote{Enrico Piaggio}, University of Pisa, 56122, Pisa, Italy.}
}

\maketitle

\begin{abstract}
    Characterizing the risk of operations is a fundamental requirement in robotics, and a crucial ingredient of safe planning.
    The problem is multifaceted, with multiple definitions arising in the vast recent literature fitting different application scenarios and leading to different computational approaches.
    A basic element shared by most frameworks is the definition and evaluation of the probability of collision for a mobile object in an environment with obstacles.
    We observe that, even in basic cases, different interpretations are possible.

    This paper proposes an index we call \enquote{Risk Density}, which offers a theoretical link between conceptually distant assumptions about the interplay of single collision events along a continuous path.

    We show how this index can be used to approximate the collision probability in the case where the robot evolves along a nominal continuous curve from random initial conditions. Indeed, under this hypothesis the proposed approximation outperforms some well-established methods either in accuracy or computational cost.
\end{abstract}


\section{INTRODUCTION}
\label{sec:intro}
\input{RDintro}

\section{BACKGROUND}
\label{sec:back}
\input{RDback}

\section{PROBLEM STATEMENT}
\label{sec:statement}
\input{RDstatement}

\section{CURRENT COMPUTATIONAL APPROACHES}
\label{sec:problems}
\input{RDproblems}

\section{RISK DENSITY}
\label{sec:approach}
\input{RDsetup}

\section{NUMERICAL COMPARISON}
\label{sec:numval}
\input{RDdiscussion}

\section{DISCUSSION}
\label{sec:discussion}
\input{RDdiscussion2}

\section{RISK DENSITY AND THE SENSITIVITY OF COLLISION PROBABILITY}
\label{sec:newapprox}
\input{RDnewapprox}

\section{TRAJECTORY OPTIMIZATION}
\label{sec:trajopt}
\input{RDopt}

\section{CONCLUSIONS}
\label{sec:conclusion}
\input{RDconclusions}

\section{ACKNOWLEDGEMENTS}
\label{sec:ack}
\input{RDack}

\appendix
\section{A}
\label{app:A}
\input{RDappendix}

\bibliographystyle{IEEEtran}
\bibliography{main}

\end{document}

%% file: RDintro.tex
\IEEEPARstart{A}{s} more and more autonomous systems are asked to perform safety-critical tasks, the importance of evaluating the \emph{risk} associated with a set of actions of a robot is becoming paramount. Several applications, including warehouse management, autonomous driving, medical and cooperative robotics, require the system not only to be \emph{safe} at any instant in time, but to plan actions that assure safety also in the future \cite{fraichardShortPaperMotion2007}. Uncertainty plays a central role in the process of associating a measure of how problematic planned actions could be. Indeed, what the agent knows about itself and its environment drastically affects the potential risk of its plan.

A centerpiece of risk evaluation for mobile robots is the ability to give an accurate estimate of \enquote{the probability of the robot not being able to finish the path} \cite{xiaoRobotRiskAwarenessFormal2020} . Indeed, this definition requires, even in the simplest case, the knowledge of the probability of collision of the robot, be it in the continuous or the discrete case \cite{masahiroonoIterativeRiskAllocation2008}, \cite{hanNonGaussianRiskBounded2022}.

\begin{figure}
  \centering
  \begin{subfigure}{0.6\columnwidth}
    \centering
    \includegraphics[clip,width=\textwidth]{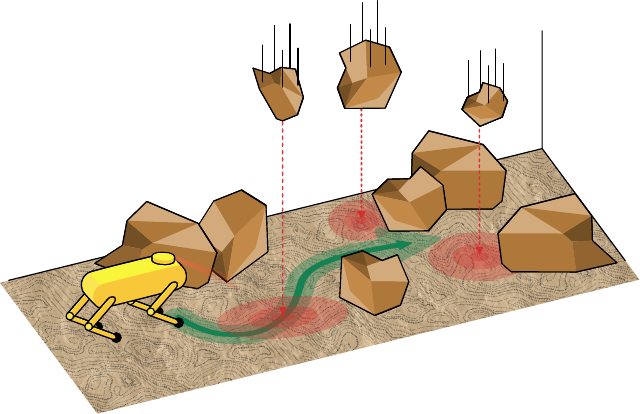}
    \caption{Hazardous mine operation.\label{sfig:mine-H1}}
  \end{subfigure}
  \begin{subfigure}{0.6\columnwidth}
    \centering
    \includegraphics[clip,width=\textwidth]{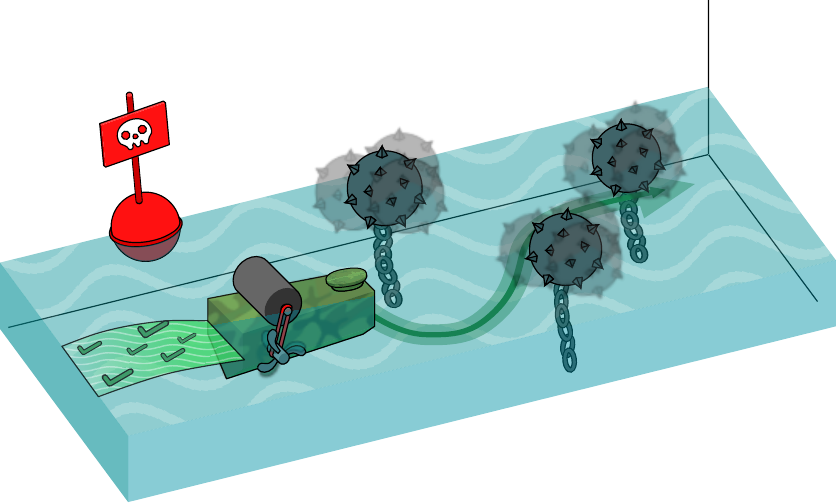}
    \caption{Mined sea.\label{sfig:sweeper-H2}}
  \end{subfigure}
  \begin{subfigure}{0.6\columnwidth}
    \centering
    \includegraphics[clip,width=\textwidth]{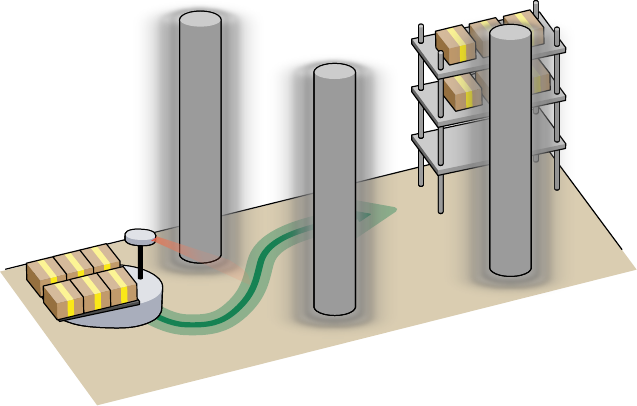}
    \caption{Unmapped warehouse.\label{sfig:delivery-H3}}
  \end{subfigure}
  \caption{Depiction of different scenarios where collisions have to be modeled in different ways. (a) shows an environment with areas where falling rocky debree could hit the robot at any point in time. (b) shows a robot boat navigating in a randomly mined body of water. (c) shows a warehouse robot navigating an uncertain environment.\label{fig:scenarios}}
\end{figure}

This estimation problem is still not considered solved, as all the methods found in literature, to the best of the authors' knowledge, fail to give a good estimate of this quantity while capturing all the aspects of the problem at hand.  Specifically, approximation schemes fail to assess the true value of the probability or are computationally prohibitive and usually depend on parameters foreign to the problem setting.
Moreover, even the very definition of the probability of a continuous set of events (in our case, the collision probability along a path) is not univocally defined without further assumptions about the process and the interaction between events happening at different configurations, as shown in \Cref{fig:scenarios}.

The literature reports multiple methods to evaluate event probabilities, each with merits and limitations.
Monte Carlo simulations \cite{blackmoreProbabilisticParticleControlApproximation2010,lambertFastMonteCarlo2008,owenMonteCarloTheory2013} can accurately estimate collision probability, but are computationally intensive since estimating the desired distribution requires a large number of samples.
Moreover, the discrete and \enquote{batch} nature of Monte Carlo methods is at odds with a continuous optimization framework \cite{freyCollisionProbabilitiesContinuousTime2020}.
Approaches to lower the computational hardness of this method are possible: e.g., \cite{jansonMonteCarloMotion2015} uses importance sampling techniques to make the simulations satisfy real-time requirements.

Chance constraints (CC) \cite{hanNonGaussianRiskBounded2022,dutoitProbabilisticCollisionChecking2011,blackmoreConvexChanceConstrained2009,shapiroLecturesStochasticProgramming2009} indicate imposed restrictions on the states which have to hold in a probabilistic sense.
CCs can concern either single events or, more commonly, joint sets of these.
Joint chance constraints (JCC) \cite{Hessem2004StochasticIC} then require an evaluation of the probability of the set of events which, in our case, is the set of collision events along a path.
Unfortunately, JCC generally have to be approximated to be used, as their interdependent nature makes them hard to evaluate either computationally or analytically.

JCC evaluation can be computed either at once, or sequentially. In both cases, several techniques have been developed to tackle this evaluation.
In the first category, we find Ellipsoidal approximations \cite{Hessem2004StochasticIC} that simplify the multivariate integral of the joint constraint to be one-dimensional, offering good computational performance, but approximating the actual value in a very conservative way.
Instead of evaluating the joint constraint at once, sequential methods consist in subdividing the JCC into individual CCs and combining their appraised value.
Indeed, in this case, the approximation to be employed is two-fold: both the value of each singular CC and their mutual interaction are simplified for computation.
In practice, either Boole's Lemma is used to relax the joint constraint into a sum of multiple simpler scalar chance constraints \cite{blackmoreProbabilisticApproachOptimal2006}, or the constraints are considered independent variables and combined in a multiplicative way. Individual chance constraints can be translated to deterministic ones using the approximations used for joint constraints or by using more ad-hoc approaches \cite{dutoitProbabilisticCollisionChecking2011,parkEfficientProbabilisticCollision2016,parkFastBoundedProbabilistic2018,parkFastBoundedProbabilistic2020}. The additive relaxation result of Boole's Lemma is known to be conservative. However, techniques such as risk-allocation \cite{masahiroonoIterativeRiskAllocation2008,blackmoreChanceConstrainedOptimalPath2011} can mitigate this by assigning a larger \enquote{probability mass} to stages of the path where the collision is more likely.

Occupancy grids methods \cite{xiaoRobotRiskAwarenessFormal2020,agha-mohammadiSMAPSimultaneousMapping2016,elfesUsingOccupancyGrids1989} are another prevalent choice used for obstacle avoidance.
Such approaches discretize the robot workspace in cells and assign an occupancy probability value to each cell to perform risk evaluation.
Grid-based methods, by considering each cell of the grid independent from one another, use a multiplicative relaxation of the cumulative probability. While being the most popular choice to compute the collision probability along a path, grid-based approaches are limited. Indeed, choosing a discretization mesh severely influences the computed value of the collision probability \cite{laconteLambdaFieldContinuousCounterpart2019}.

All the previous approaches, except for Monte Carlo sampling, work well for linear problems and Gaussian random variables. More general approaches exist, able to tackle both nonlinear systems and more general random variables via Lasserre Hierarchy \cite{lasserreMomentSOSHierarchy2018}, \cite{hanNonGaussianRiskBounded2022}, but these are computationally prohibitive.
The language of exit times is used in \cite{freyCollisionProbabilitiesContinuousTime2020} to derive an estimate in a continuous and nonlinear setting, but considers only the robot uncertainty.

\subsection{Contribution}
Our contribution is to illustrate some of the problems encountered in applying Grid-Based and JCCs methods and to propose a novel method to characterize and approximate the robot collision probability along a continuous trajectory. To do so, we focus on a well-defined case with deterministic dynamics, where a rigorous definition of the collision probability along a continuous path can be given.
The proposed approximation shows better or comparable accuracy to other approaches while being computationally competitive. This approximation depends on the objects' combined dimension and a quantity, which we call risk density, that depends only on the parametrization of the path followed by the robot, the obstacle position, and no exogenous parameters. Additionally, this metric gauges the continuous path as a whole instead of subdividing it into waypoints.
We stress that our intent is both to highlight the existence of the problem and to give computationally tractable tools to tackle it.

The paper is organized as follows. \Cref{sec:back} explores how the probability of intersection between two uncertain shapes can be formalized and computed. In \Cref{sec:statement} we discuss how this formalization translates to the problem of collision probability computation along a continuous path. In this section, we discuss some possible assumptions about the interdependence of events and define the probability of collision as a consequence; secondly, we discuss an example of how this probability can be computed and highlight some issues due to the continuity of the problem.
\Cref{sec:problems} briefly describes current approaches and their limitations.
In \Cref{sec:approach}, we propose the \emph{Risk Density} index, briefly discuss its relevance and introduce an approximation to the collision probability along a continuous path. In \Cref{sec:numval}, we display the computational experiments we have carried out as validation and compare our approximation performance to other techniques. In \Cref{sec:discussion} we discuss the properties of the approach and how these are revealed in the benchmark proposed while also highlighting the limitations of the method.
In \Cref{sec:newapprox} we show how the same \emph{Risk Density}, can be used to correct a previously computed Monte Carlo estimate, after the simulation setting has been altered.
Finally, in \Cref{sec:trajopt} some applications of \emph{Risk Density} in a trajectory optimization setting are presented.
\Cref{sec:conclusion} concludes the paper.

%% file: RDback.tex
Consider a robot and an obstacle immersed in a workspace $\mathbb{W} \subseteq \mathbb{R}^n$.
Let $\mathbb{X}_R(x_R)\subseteq \mathbb{W}$ be the set of points occupied by the robot and $\mathbb{X}_O(x_O)\subseteq \mathbb{W}$ the set occupied by the obstacle, where $x_R, x_O \in \mathbb{W}$ are the reference frames locations for the robot and obstacle, respectively. 

The overlap condition \cite{lavallePlanningAlgorithms2006,dutoitProbabilisticCollisionChecking2011} defined as 
\begin{equation}
    \label{eq:collisionevent}
    C^{[0]}(x_R,x_O) : \mathbb{X}_R(x_R) \cap \mathbb{X}_O(x_O) \ne \emptyset\,,
\end{equation}
indicates wherever the two sets of points intersect, or in other words, that a collision has happened if we consider a dynamic scenario.

The GJK algorithm \cite{gilbertFastProcedureComputing1988} is a standard practical tool to check \cref{eq:collisionevent} when all the information about the shape and position of both the robot and object is wholly known. 
One of its building blocks is the notion of Minkowski sum on sets,
\begin{equation}
    \label{def:minkadd}
    A\oplus B := \{a+b | a\in A, b\in B   \},
\end{equation}
which can be used to define an equivalent condition to \cref{eq:collisionevent}, which is
\begin{equation}
    \label{eq:minkcollision}
    C^{[0]}(x_R,x_O) : \{\overline{0}_n\} \subseteq \mathbb{X}_R(x_R) \ominus \mathbb{X}_O(x_O),
\end{equation}
where $\overline{0}_n\in \mathbb{R}^n$ is the $n$-dimensional origin and $A\ominus B \triangleq A \oplus (-B)$ is sometimes known as Minkowski difference. 

Unfortunately, the GJK algorithm is rendered useless as soon as uncertainty is introduced in the problem. 
In fact, even assuming that both the obstacle and the robot are perfectly known in shape but uncertain in position, condition \cref{eq:collisionevent} cannot be checked in a deterministic way anymore, and we are required to define and appraise an overlap/collision probability.

A possible way to define such a probability is to resort to  an \emph{indicator function}, defined from \eqref{eq:collisionevent} as
\begin{equation}
    \label{eq:indicatorcollision}
    I_{C^{[0]}}(x_R,x_O) =
    \begin{cases}
        1 & \text{if } \mathbb{X}_R(x_R) \cap \mathbb{X}_O(x_O) \ne \{ \emptyset \}, \\
        0 & \text{otherwise.}
    \end{cases}
\end{equation}
Now, if we consider $x_R$ and $x_O$ to be realizations of two random variables, $R$ and $O$, characterized by a joint probability density function $f_{R,O}(r,o)$, the overlap or collision probability is
\begin{equation}
    \label{eq:probcollisionevent}
    P(C^{[0]}) = \int_{\mathbb{W}} \int_{\mathbb{W}} I_{C^{[0]}}(r,o) f_{R,O}(r,o) drdo\,.
\end{equation}

While \eqref{eq:probcollisionevent} gives the actual value of $P(C^{[0]})$, its numerical computation is, in general, not trivial as it consists of two multidimensional integrals in $\mathbb{R}^n$, where at each point $(r,o)$ one should calculate $I_{C^{[0]}}$, e.g. via GJK.
However, that analysis can be simplified by resorting to \eqref{eq:minkcollision} and shifting to another perspective.
In a similar fashion to \cite{parkEfficientProbabilisticCollision2016}, we decompose the uncertain sets $\mathbb{X}_R$ and $\mathbb{X}_O$ as
\begin{equation}
    \label{eq:setdecompositionR}
    \mathbb{X}_{R} = S_{R} + d_{R} + N_{R}(\omega),
\end{equation}
\begin{equation}
    \label{eq:setdecompositionO}
    \mathbb{X}_{O} = S_{O} + d_{O} + N_{O}(\omega),
\end{equation}
as shown in \Cref{fig:decompositionuncertainset}.
Now $S_R\subseteq \mathbb{W}$ and $S_O\subseteq \mathbb{W}$ are sets centered in the origin of the deterministic shapes, $d_R,\ d_O \in \mathbb{W}$ are vectors equal to the mean of $R$ and $O$ respectively, while $N_R$ and $N_O$ are zero-mean Random Variables such that $R=d_R+N_R$, $O=d_O+N_O$.

Condition (\ref{eq:minkcollision}) becomes equivalent to
\begin{flalign}
         & \{0\} \subseteq \mathbb{X}_R(R) \ominus \mathbb{X}_O(O)\nonumber  &                           \\
    \iff & \{0\} \subseteq (S_R\ominus S_O) + (d_R-d_O) + (N_R-N_O)\nonumber &                           \\
    \iff & \{0\} \subseteq S_{RO} + d_{RO} + N_{RO}\nonumber                 &                           \\
    \iff & N_{OR} \in S_{RO} + d_{RO}                                        & \label{eq:collmovingset}  \\
    \iff & d_{OR}+N_{OR} \in S_{RO},                                         & \label{eq:collmovingdist}
\end{flalign}
where 
\begin{equation}
    d_{RO} =  d_R - d_O = -d_{OR},  \label{dRO}
\end{equation}
is the distance vector between the mean of the two random variables and
\begin{equation}
    N_{RO} = N_R - N_O = -N_{OR},  \label{NRO}
\end{equation}
is a zero-mean Random Variable describing the uncertainty of such a vector.

\begin{figure}  
    \centering
    \includegraphics[clip,width=0.8\linewidth]{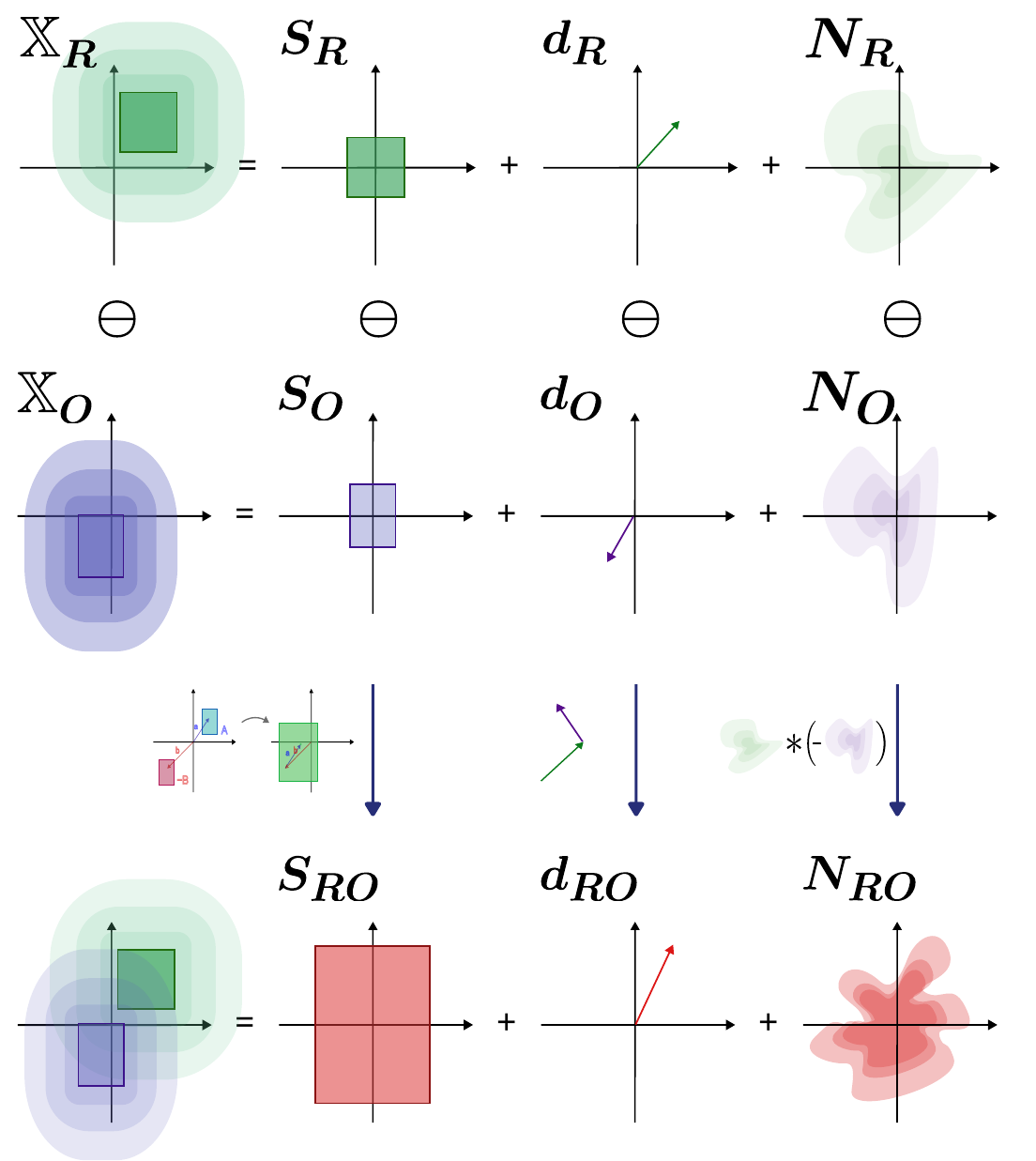}
    \caption{Graphical depiction of the decomposition of the robot and obstacle sets. The result of the Minkowski sum is also illustrated. Each object is represented by
 its shape, the vector pointing to its frame of reference, and a Random Variable describing the uncertainty in its position.  \label{fig:decompositionuncertainset}}
\end{figure}
Using conditions (\ref{eq:collmovingset}) and (\ref{eq:collmovingdist}), both derived from the set decomposition, (\ref{eq:probcollisionevent}) can be restated in as either
\vspace{-3mm}
\begin{align}
     & P(C^{[0]}) & = & \int_{S_{RO}} (f_{N_{OR}}(x)*\delta_{d_{OR}}(x))dx \nonumber                \\
     &      & = & \int_{S_{RO}} f_{N_{OR}}(x+d_{OR})dx \label{eq:probmovingdist}              \\
     & \mbox{or} &&\nonumber    \\
     & P(C^{[0]}) & = & \int_{S_{RO} + d_{RO}} f_{N_{OR}}(x)dx\,.             \label{eq:probmovingset}
\end{align}
Here $f_{N_{OR}}=f_{N_O}*f_{N_{-R}}$ is the probability density function resulting from the difference $N_O-N_R$, while $\delta_{d_{OR}}(x)$ is the Dirac delta function centered in $d_{OR}$.
Both \cref{eq:probmovingdist} and \cref{eq:probmovingset} are computed via a single integral on a defined set in contrast to \cref{eq:probcollisionevent}. 
All the remaining complexity hides inside $f_{N_{OR}}$ since, in general, computing this pdf requires a convolution operation.

This may not be an issue in the case when the convolution has closed form expression, e.g. if $R$ and $O$ are normally distributed as $R\sim\mathcal{N}(\mu_R,\Sigma_R)$, $O\sim\mathcal{N}(\mu_O,\Sigma_O)$. In that case \cref{eq:probmovingdist} and \cref{eq:probmovingset} simplify \cite{pishro-nikProbabilityStatisticsRandom2014} to
\begin{align}
     & P(C^{[0]}) & = & \int_{S_{RO}} \mathcal{N}(x|\mu_O-\mu_R,\Sigma_T)dx \label{eq:probmovingdistgauss}                     \\
     & P(C^{[0]}) & = & \int_{S_{RO} + (\mu_R - \mu_O)} \mathcal{N}(x|0,\Sigma_T)dx,             \label{eq:probmovingsetgauss}
\end{align}
where $\Sigma_T=\Sigma_R+\Sigma_O$.

%% file: RDstatement.tex
Having formalized the probability of a collision event in a fixed configuration in \Cref{sec:back}, we now approach the problem of defining the same quantity along a path followed by the robot. A path is a function $\mu_R(s):\mathbb{S}\to\mathbb{W}$ going from a possibly continuous ordered set of configurations $\mathbb{S}=\{s_0,s_1,\dots,s_N\}$ to the workspace.
\\
Denoting $\mathcal{W}$ as a sufficiently well-behaved family of sets in $\mathbb{W}$\footnote{For example the Borel algebra of $\mathbb{W}$, $\mathcal{B}(\mathbb{W})$.}, i.e. $w_i \in \mathcal{W} \implies w_i \subseteq \mathbb{W} $, the integration set associated with each configuration $s$ is given by the function $D_{RO}(s) : \mathbb{S} \to \mathcal{W}$,
\begin{equation}
  \label{eq:integrationsetconf}
  D_{RO}(s) = S_{RO} + (\mu_R(s)-\mu_O).
\end{equation}
We also explicitly consider the dependence on $s$ of $N_{RO}$ on $s$ by treating it as a stochastic process $N_{RO}(\omega, s) : \Omega \times \mathbb{S} \to \mathbb{W}$, where $\Omega$ is the sample space.
Similarly to condition \cref{eq:collmovingset}, we associate each configuration $s$ to the stochastic process $C(\omega, s) : \Omega \times \mathbb{S} \to \{0,1\}$,
\begin{equation}
  \label{eq:stochprocesscollision}
  C(\omega,s) : N_{RO}(\omega,s) \in D_{RO}(s),
\end{equation}
and, with a slight abuse of notation, we parametrize the stochastic collision event as
\begin{align}
  \label{eq:collisionparametrized}
  C_s & := \{ \omega \in \Omega | C(\omega,s) = 1, \ s\in \mathbb{S} \} \nonumber        \\
      & = \ \{ \omega \in \Omega | N_{RO}(\omega,s) \in D_{RO}(s), \ s\in \mathbb{S} \},
\end{align}
i.e. the set of all possible outcomes which result in a collision in $s$.
Consequently each configuration has an associated probability in the form of \cref{eq:probmovingsetgauss} as
\begin{equation}
  \label{eq:probconfigurationmoving}
  P(C_s)  =  \int_{D_{RO}(s)} \mathcal{N}(x|0,\Sigma_T)dx.
\end{equation}

In general both the robot and the obstacle dynamics can be stochastic, making the random process $N_{RO}(\omega,s)$ at each $s$ correlated to the history of past configurations. In general, to model the dynamics of such a process, the tools of It\^o Calculus \cite{speyerStochasticProcessesEstimation2008} are to be used. However in this paper we assume that the stochastic process is identically distributed, i.e. $N_{RO}(\omega,s_i)\sim \mathcal{N}(0,\Sigma_T)$.

While \cref{eq:probconfigurationmoving} is valid for an isolated configuration $s$, interactions between neighboring configurations in $\mathbb{S}$ must be accounted for if a path is considered.

Indeed what we are interested in is the value of
\begin{equation}
  \label{eq:totalprobabilityconfigurations}
  P(C) := P\left(\bigvee_{\mathbb{S}}C_s\right),
\end{equation}
which however is not directly computable due to the non trivial dependence between events \eqref{eq:collisionparametrized}.

\subsection{Assumptions}
\label{subsec:assumptions}

To model such interactions, and to compute \cref{eq:totalprobabilityconfigurations}, it is necessary to specify further hypotheses which apply to different cases.
An assumption that is commonly adopted is to consider that all collision events are independent
(e.g. \cite{vandenbergLQGMPOptimizedPath2011}).
Formally, this model can be described as
\begin{enumerate}
  \item[\textbf{(H1)}] $C_{s_i} \ \bot \ C_{s_j}, \ \forall s_i\neq s_j \in \mathbb{S},$
\end{enumerate}
leading to the following lemma.

\begin{lemma}[\textbf{H1}]
  \label{def:H1}
  The probability of collision \cref{eq:totalprobabilityconfigurations} under (H1) reduces to
  \begin{equation}
    P_{H1}(C)  =  1 - \prod_{\mathbb{S}}(1-P(C_s)), \label{eq:probH1},
  \end{equation}
  where $P(C_s)$ is \cref{eq:probconfigurationmoving}.
\end{lemma}
\begin{proof}
  \begin{align}
                                                                                        & P\left(\bigvee_{\mathbb{S}}C_{s}\right) = P \left(\overline{\bigwedge_{\mathbb{S}}\overline{C}_{s}}\right)=
    1 - P \left(\bigwedge_{\mathbb{S}}\overline{C}_{s}\right)\xrightarrow[\text{ind}]{} & \nonumber                                                                                                     \\
                                                                                        & \xrightarrow[\text{ind}]{} 1 - \prod_{\mathbb{S}}(1-P(C_s)),                                                &
    \label{eq:proofH1}
  \end{align}
\end{proof}
While this premise is reasonable in the discrete case, it is easy to show that it leads to incongruities as we approach a continuous path. As pointed out by \cite{jansonMonteCarloMotion2015}, the assumption is \emph{asymptotically tautological}, as the estimation of the collision probability resulting from (H1)  converges to $1$ with the cardinality of the set $|\mathbb{S}|\to \infty$. Hence with increasing the path discretization resolution, all paths would almost surely lead to collision.

To fix this issue (H1) needs to be relaxed. A less restrictive alternative to the independence hypothesis is the assumption that the collision event is Markovian: the events $C_s$ depend only on the immediately preceding past, i.e. $P(C_{s_j}|\overline{C}_{s_{j-1}},\dots,\overline{C}_{s_0})=P(C_{s_{j-1}}|\overline{C}_{s_j})$. In this case the conditioned probability is analytically unapproachable in general, and approximations about its formulations are to be introduced. For instance, the authors of \cite{patilEstimatingProbabilityCollision2012} use truncated Normal distributions to model preceding configurations being collision-free. We introduce here a different simplification: we consider that the collision could only happen on newly swept \enquote{area} in $\mathbb{W}$ at each index $s_i$. More precisely we assume that

\begin{itemize}
  \item [\textbf{(H2)}]
  \item [$\bullet$] $C_s$ has the Markov property
  \item [$\bullet$] $P(C_{s_{i+1}}|\overline{C}_{s_{i}}) = P(N_{RO}(s_{i+1})\in(D_{RO}(s_{i+1})\cap D_{RO}^C(s_{i})))$, where $D_{RO}^C(s)$ denotes the complement of the set $D_{RO}(s)$.
\end{itemize}

\begin{lemma}[\textbf{H2}]
  \label{def:H2}
  The probability of collision \cref{eq:totalprobabilityconfigurations} under (H2) reduces to
  \begin{equation}
    P_{H2}(C)  =  1 - \prod_{s_i\in \mathbb{S}}(1-P(C_{s_{i+1}\cap \overline{s}_i})), \label{eq:probH2}
  \end{equation}
  where
  \begin{align}
     & C_{s_{i+1}\cap \overline{s}_i} := \nonumber \\ &\{ \omega \in \Omega | N_{RO}(\omega,s_{i+1})\in(D_{RO}(s_{i+1})\cap D_{RO}^C(s_{i})) \}.
    \label{eq:disjointH2}
  \end{align}
\end{lemma}

\begin{proof}
  \begin{align}
                                                                                         & P\left(\bigvee_{\mathbb{S}}C_{s}\right) = P \left(\overline{\bigwedge_{\mathbb{S}}\overline{C}_{s}}\right)=
    1 - P \left(\bigwedge_{\mathbb{S}}\overline{C}_{s}\right)\xrightarrow[\text{mark}]{} & \nonumber                                                                                                                                            \\
                                                                                         & \xrightarrow[\text{mark}]{} 1 - \prod_{s_i\in\mathbb{S}}(1-P(C_{s_{i+1}}|\overline{C}_{s_{i}}))\stackrel{\text{\cref{eq:disjointH2}}}{=} & \nonumber \\
                                                                                         & \stackrel{\text{\cref{eq:disjointH2}}}{=} 1 - \prod_{s_i\in \mathbb{S}}(1-P(C_{s_{i+1}\cap \overline{s}_i})).                            & \nonumber
    \label{eq:proofH2}
  \end{align}
\end{proof}

\begin{figure}[]
  \centering
  \includegraphics[clip,width=0.9\linewidth]{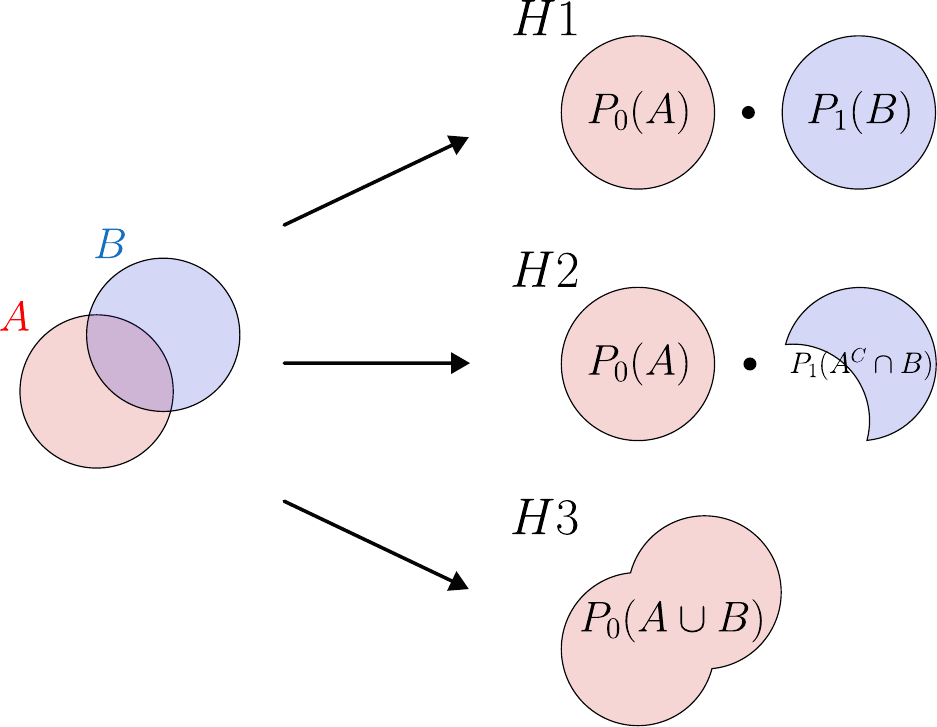}
  \caption{Visual representation of how considered assumptions work on two subsequent events. The events considered are $\{\omega | N_{RO}(s_0,\omega)\in A\}$ and $\{\omega | N_{RO}(s_1,\omega) \in B\}$, where the sets $A$ and $B$ are shown as red and blue disks. $P_0(\cdot)$ and $P_1(\cdot)$ stand respectively for the probability measure of the first and the second random variable to be inside a set.}
  \label{fig:H1H2H3}
\end{figure}
\vspace{-2mm}
An alternative model underpinning a different path integral extension of the collision probability formula stems from considering the evolution of the robot path as completely deterministic, while limiting the uncertainty to the sole knowledge of the initial obstacle and robot position.
This premise can be formally stated using the concept of stopped process \cite{gallagerStochasticProcessesTheory2013}.

\begin{enumerate}
  \item[\textbf{(H3)}] $N_{RO}(\omega,s)$ is a stopped  process  $N^{s_0}_{RO}(\omega,s) := N_{RO}(\omega,\min(s_0,s))$.
\end{enumerate}

In this case the stochastic process is constant, as $s_0$ is the first element of $\mathbb{S}$, meaning that and all the uncertainty can be represented as a single Random Variable $N^{s_0}_{RO}(\omega)$ sampled from a sample space $\Omega$.

We then introduce \cref{eq:totalprobabilityconfigurations} under (H3):
\begin{lemma}[\textbf{H3}]
  \label{def:H3}
  The probability of collision \cref{eq:totalprobabilityconfigurations} under (H3) reduces to
  \begin{equation}
    P_{H3}(C)  =  \int_{D_T} \mathcal{N}(x|0,\Sigma_T)dx, \label{eq:totalprobconfigurationmoving}
  \end{equation}
  with
  \begin{equation}
    D_T = \bigcup_{\mathbb{S}} D_{RO}(s).\label{eq:DROT}
  \end{equation}
\end{lemma}

\begin{proof}
  $\bigvee_{\mathbb{S}}C_s = \bigvee_{\mathbb{S}} \{ N_{RO}(\omega,s) \in D_{RO}(s) \} = \bigvee_{\mathbb{S}} \{ N_{RO}(\omega,s_0) \in D_{RO}(s) \} = \{ N_{RO}(\omega,s_0) \in \bigcup_{\mathbb{S}} D_{RO}(s) \} $,
  of which the probability is \cref{eq:totalprobconfigurationmoving}.
  \label{eq:proofH3}
\end{proof}

Note that \cref{def:H3} is a specialization of \cref{eq:probmovingsetgauss} on the set $D_T$.
Moreover, notice how by swapping $D_T$ with the union of the $D_{RO}(s)$ over any contiguous subset of $\mathbb{S}$, i.e. $D_{RO}(s_1,s_2) = \bigcup_{s=s_1}^{s_2} D_{RO}(s)$, we obtain the probability of the related subpath. In this sense the \enquote{Probability To Go} is computed by integrating \cref{eq:totalprobconfigurationmoving} over $D_{RO}(s,s_N)$.

While hypotheses such as (H1), (H2) and (H3), shown in \cref{fig:H1H2H3}, are mathematical abstractions, these tightly connect to the assumptions one can draw from the particular scenario.
Indeed, (H1) can model fairly well an obstacle whose position changes over time with a dynamics that is much faster than that of the robot, but occupies, on average, the same regions of space with a fixed frequency: think of obstacles such as a stream of people, or falling debris, or droplets of water.
On the other hand (H3) would more accurately model a scenario where, for example, the obstacle is a pillar whose position is fixed but uncertain.
Finally, hypothesis (H2) is the more involved assumption. This hypothesis models scenarios where collisions can occur only in the additional area covered at the last step taken. For example, consider a robot boat navigating through a minefield in the sea. It is reasonable to assume that collisions can only happen in front of the boat. However, a return trip on the same path would not be automatically safe because the mines bob with the waves and slowly change their, already uncertain, positions around their respective means.
A stylized depiction of a few of these example scenarios is shown in \Cref{fig:scenarios}.

Notice that the complete set of possible scenarios form a somehow continuous spectrum in which (H1) and (H3) lie at two opposite extremes: (H1) is total independence between configurations, while (H3) denotes \enquote{total dependence}. Multiple different modelling assumptions about the uncertainty can be found in the range in between (H1) and (H3), such as (H2).
We show in simulation in \Cref{sec:numval} how these hypotheses influence probability estimates in the same simple scenario.
\begin{figure}
  \centering
  \includegraphics[clip,width=0.9\linewidth]{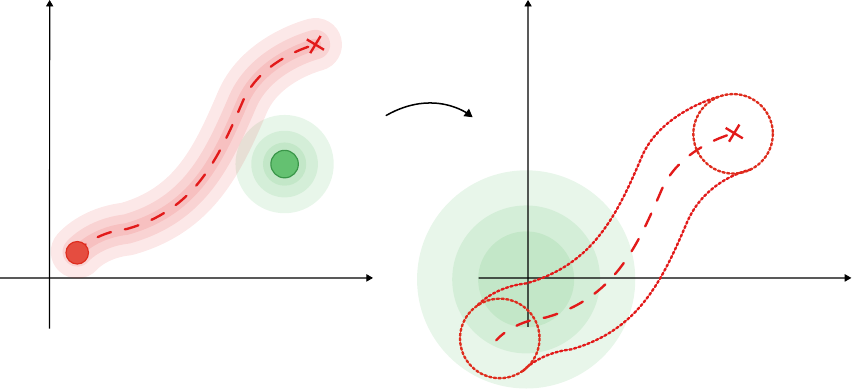}
  \caption{Interpretation of \cref{eq:totalprobconfigurationmoving}. The problem is moved from the initial domain $\mathbb{W}$ to the domain of the difference between the robot and the obstacle shape. This results from the transformation shown in \cref{fig:decompositionuncertainset}, which moves all the uncertainty on a single distribution centered in zero, while the integration set is $D_T$.\label{fig:tubedifference}}
\end{figure}
\vspace{-2mm}
\subsection{Continuous Path Parametrization}
In the case in which the set $\mathbb{S}$ is continuous, one has
either to evaluate a union of the uncountable number of sets spanned by the robot moving along the trajectory, \cref{eq:DROT}, or to compute the product over an uncountable number of indices. Computing the union or the product turns out to be unpractical, or even impossible in closed form. The problem is, for a path parametrized as $\mu_R(s): [0,1] \to \mathbb{W}$, to find an  expression of ${D}_{RO}(s)$ and $D_T$ \cref{eq:DROT}.
For simplicity let us assume henceforth that $\mu_R(s)$ is almost everywhere differentiable map and $\mathbb{W} = \mathbb{R}^2$. Consider two disks $S_R$ and $S_O$ outer-bounding the shape of the objects.
Then, the combined shape $S_{RO}$ is a disk whose radius is the radii sum.
Here $D_T$ takes the shape of a rounded tube, drawn by the large disk being swept along the trajectory, as shown in \cref{fig:tubedifference}.

The characterization of such a swept space is generally rather complex.
Recently, that characterization has been addressed by \cite{demont-marinMinimumSweptvolumeMetric2022} with nontrivial mathematical methods, and focusing on local aspects.

Suppose a fully diffeomorphic map can parameterize the swept space
\begin{equation}
  \label{eq:generaltubeapprox}
  \Phi(s,\theta):[0,1]\times[-T,T]\to \mathbb{W}\,,
\end{equation}
where $s\in [0,1]$ is a line variable defined along the path $\mu_R(s)$ through which the disk is dragged, and $\theta\in[-T,T]$ is a characteristic length associated with the shape of the robot. Now, using \cref{eq:generaltubeapprox},  \cref{eq:totalprobconfigurationmoving} becomes
\begin{equation}
  \label{eq:probabilityRVparam}
  P^a_{H3}(C)=\int_0^1\int_{-T}^{T} \mathcal{N}(\Phi(\gamma,\nu)-\mu_O|0,\Sigma_T)\Gamma(\gamma,\nu) d\nu d\gamma\;,
\end{equation}
where $\Gamma(\gamma,\nu) = |\det (\nabla\Phi(\gamma,\nu))|$.

As explored in \cite{paiolaConsiderationsPossibleApproaches2022}, \cref{eq:generaltubeapprox} can be used also to compute \cref{eq:probH2} under (H2). Indeed, using the parametrization and adhering to \cref{eq:disjointH2}, the conditional probability in the continuous case can be rewritten as
\begin{align}
   & P^a(C_{s+ds}|\overline{C}_s) \nonumber                                                                                \\
   & = \int_s^{s+ds} \int_{-T}^{T} \mathcal{N}(\Phi(\gamma,\nu)-\mu_O|0,\Sigma_T)\Gamma(\gamma,\nu) d\nu d\gamma \nonumber \\
   & = \int_{-T}^{T} \mathcal{N}(\Phi(s,\nu)-\mu_O|0,\Sigma_T)\Gamma(s,\nu) d\nu ds.
  \label{eq:conditionalslice}
\end{align}
Plugging this in \cref{eq:probH2} results in a Volterra integral \cite{slavikProductIntegrationIts2007}, leading to
\begin{align}
   & P^a_{H2}(C) = \nonumber \\ &1-\exp(-\int_0^1\int_{-T}^{T} \mathcal{N}(\Phi(\gamma,\nu)-\mu_O|0,\Sigma_T)\Gamma(\gamma,\nu) d\nu d\gamma)
  \label{eq:volterraprob}
\end{align}

The problem with this approach is to find a parametrization diffeomorphic to $\mathbb{W}$, which is generally not trivial.

Consider indeed the na\"ive parametrization
\begin{equation}
  \label{eq:tubeparametrizationexample}
  \Phi(s,\theta) = \begin{bmatrix}
    \mu_{R_x}(s) - \theta \frac{d\mu_{R_y}}{ds}(s) \\
    \mu_{R_y}(s) + \theta \frac{d\mu_{R_x}}{ds}(s)
  \end{bmatrix}.
\end{equation}
In this case the curve parametrization $\mu_R(s)$ is augmented by the variable $\theta$, which indicates the \enquote{off-track} parameter, i.e. the locally orthogonal deviation from $\mu_R(s)$. The boundary of the domain of $\theta$ is given by the parameter $T$, which for the space swept by an object of circular shape amounts to the object radius $r$ (in our case, as the circle is generated by two objects, $r=r_r+r_O$).

Note that \cref{eq:tubeparametrizationexample} is not a diffeomorphic map between the parametrization space and $\mathbb{W}$, indeed it is only homeomorphic locally in a neighborhood around $T=0$, unless the trajectory is a straight line. Therefore, whenever $\mu_R(s)$ intersects itself or its curvature is not null, some region of the integration set $D_T$ will be considered multiple times\footnote{\Cref{eq:probabilityRVparam} and \cref{eq:volterraprob} only require $\Phi(s,\theta)$ to be differentiable almost everywhere w.r.t. to $s$ and $\theta$.}.
Due to those issues, the double integral \cref{eq:probabilityRVparam} does not always give a good approximation of \cref{eq:totalprobconfigurationmoving}.
Besides that, evaluating \cref{eq:probabilityRVparam} can still be relatively expensive to compute.
These problems spawn from the continuity of the path as, even under (H3), defining precisely the infinite set of unions $D_T$ for \cref{eq:totalprobconfigurationmoving} is, in general, analytically convoluted and computationally challenging.

We refer to the probability approximation computed by the integral \cref{eq:probabilityRVparam} with parametrization \cref{eq:tubeparametrizationexample} as \emph{Na\"ive Set Parametrization Approach}.

%% file: RDproblems.tex
As already discussed in \Cref{sec:intro}, different approaches in literature dealt with the estimation of collision probability along a continuous path.
Here we retrace how Grid-Based and Stage-Wise approximation methods deal with this problem, highlighting some of their drawbacks.

\subsection{Grid-Based Methods}
In their simplest form, Bayesian Grids \cite{elfesUsingOccupancyGrids1989} consist in a tessellation of the environment and a set of Random Variables $\mathbf{m}$. Each element of $\mathbf{m}$, $m_i$, is a binary Random Variable representing the occupancy of the tile in the discretized space with associated probability $P(m_i)\triangleq P(m_i=1)$.
With an implicit abuse of notation, we will indicate with $m_i$ both the occupancy of the tile, and the tile itself.
Assuming the hypothesis, similar to (H1), that each cell forming the grid is independent, enables the real-time capabilities of this approach, as the value of each single cell can be updated without considering the joint probability function of the entire map.
Specifically, the quantity updated is the belief (or posterior) probability of occupancy $P(m_i|z_{1:t},y_{1:t})$, given both the set of the robot's measurements $z_{1:t}=\{z_1,z_2,\dots,z_t\}$ and the set of its poses $y_{1:t}=\{y_1,y_2,\dots,y_t\}$ up to time $t$. The value of the cells updates according to Bayes Rule\footnote{In practice, the log-odds form of this relation is used, as it has better computational performances.}

\begin{equation}
  P(m_i|z_{1:t},y_{1:t}) = \frac{P(z_t,y_t|m_i)P(m_i|z_{1:t-1},y_{1:t-1})}{P(z_t,y_t|z_{1:t-1},y_{1:t-1})}.
  \label{eq:bayesbelief}
\end{equation}

After each cell is updated at some fixed time $t$, the independence assumption is used to compute the probability of collision along the path followed by the robot.

Denote the set of the tiles crossed by the robot as $\mathbf{m}_c$ and the collision event $m_i=1$ as $C_i$ (of which the belief is \cref{eq:bayesbelief}).
Considering independent events, the probability of collision in the grid case is
\begin{equation}
  \label{eq:gridprobbelief}
  P_g(C) =  1 - \prod_{m_i\in\mathbf{m_c}} (1- P(m_i|z_{1:t},y_{1:t})).
\end{equation}

As noted by \cite{laconteLambdaFieldContinuousCounterpart2019} and others, one of the issues of this approach is that in no way the area of the tiles is taken into consideration, as the belief probability gets updated using only the sensor readings at time $t$. Therefore, the discretization resolution influences the probability of collision computed by the method, as shown in \cref{fig:gridissue}.
\begin{figure}
  \centering
  \includegraphics[width=0.9\linewidth, trim= 0 -2mm 0 0 ]{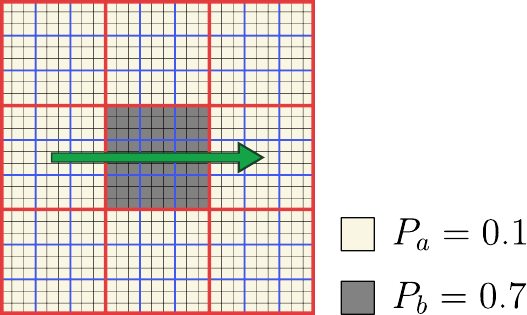}
  \begin{tabular}{l|c|c|c}
                     & \color{Red}Red                      & \color{RoyalBlue}Blue                 & \color{Black}Black                      \\
    \hline
    $|\mathbf{m_c}|$ & $3$                                 & $7$                                   & $19$                                    \\
    \hline
    $P_g(C)$         & $1-\overline{P}^2_a \overline{P}_b$ & $1-\overline{P}^4_a \overline{P}^3_b$ & $1-\overline{P}^{10}_a\overline{P}^9_b$ \\
                     & =0.757                              & =0.982\dots                           & =0.999\dots
  \end{tabular}
  \caption{Example of discretization choice effect on the computed probability by grid-based methods. The path crossed by the robot (green) lives in $\mathbb{R}^2$, and $3$ distinct tessellations of the environment are given (red, blue, black). The robot sensors measure different probability values \cref{eq:bayesbelief} for the cells depending on the image's color. The probability computed using \cref{eq:gridprobbelief} assumes different values depending on the discretization, although the underlying environment is the same. Indeed, the probability changes as a function of the cardinality of the set $\mathbf{m_c}$. \label{fig:gridissue}}
  \centering
\end{figure}
Even if we circumvent this issue and compute $P(m_i)$ instead of its belief \cref{eq:bayesbelief}, the independence assumption and the multiplicative formula \cref{eq:gridprobbelief} would lead to incorrect conclusions, as we show numerically in \Cref{sec:discussion}.

\subsection{Stage-wise Probability Estimation Methods}

\begin{figure*}
  \centering
  \begin{subfigure}{0.35\textwidth}
    \centering
    \includegraphics[clip,width=\textwidth]{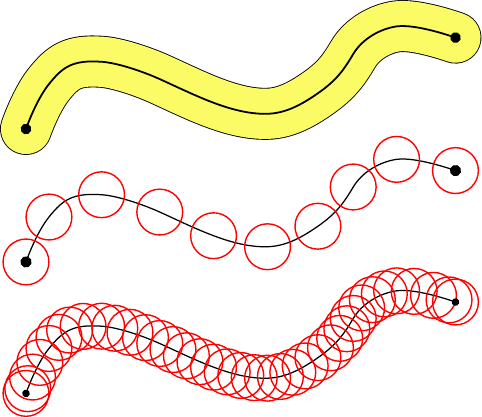}
    \caption{Joint chance constraint approximation.\label{sfig:jointcc}}
  \end{subfigure}
  \hfil
  \begin{subfigure}{0.45\textwidth}
    \centering
    \includegraphics[clip,width=\textwidth]{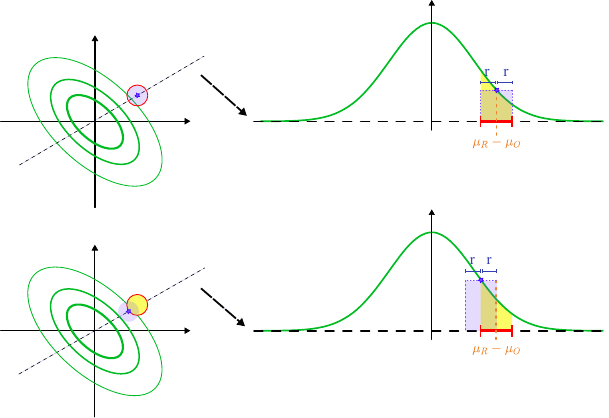}
    \caption{Probability approximations for Gaussian distributions.\label{sfig:singlecc}}
  \end{subfigure}
  \caption{Left: Example of under and over discretization of a continuous path using Boole's Lemma. Red circles represent the considered robot configurations. The swept area $D_{T}$, in yellow, is misrepresented, whether by selecting a set of configurations that is too sparse, leaving out much of the event, or too dense, parametrizing intersections multiple times. Right: Depiction of the approaches used by authors of \cite{dutoitProbabilisticCollisionChecking2011} (top) and \cite{parkFastBoundedProbabilistic2018} (bottom) to approximate Gaussian chance constraints. The boundary of the robot is in red. The approximation integration set is shaded in purple, while the true probability is in yellow.\label{fig:CCissue}}
\end{figure*}
Chance constraints are a popular framework used in optimization to address the risk of collision.
Specifically, chance constraints are inequalities involving probabilities in the form of $P(C)\leq \delta$.

To convert individual chance constraints to scalar inequalities, either an approximation or a bound on the probability of the event can be used. The authors of \cite{dutoitProbabilisticCollisionChecking2011} consider both the obstacle and the robot as normally distributed objects of spherical shape and give an approximation of the probability of collision \cref{eq:probmovingsetgauss} as
\begin{equation}
  \label{eq:dutoitapprox}
  P(C) \approx V_S \frac{1}{2 \pi \sqrt{\det(\Sigma_T)}} \exp{-\frac{1}{2}(\mu_R-\mu_O)^T \Sigma_T^{-1} (\mu_R-\mu_O)},
\end{equation}
where the variables used are the ones  introduced in \Cref{sec:back}, and $V_S$ is the area of the combined disk, in $\mathbb{R}^2$ $V_S=(r_o+r_r)^2 \pi$. The method amounts to taking the value of $\mathcal{N}(x|0,\Sigma_T)$ at $x=\mu_R-\mu_O$, the center of the combined disk obtained as described in \cref{sec:back}, and multiplying it for $V_s$.
While this approximation is efficiently evaluated and, as such, can be implemented in dynamic planning methods \cite{dutoitRobotMotionPlanning2012}, its value is not assured to upper-bound the actual probability.
An alternative to this is proposed in \cite{parkFastBoundedProbabilistic2020} where an optimization problem is set up to compute $x_{max}$, the point inside the combined disk that maximizes $\mathcal{N}(x|0,\Sigma_T)$, and then multiplying it for $V_s$: this always bounds the true probability from above, as shown in \cref{sfig:singlecc}.

Whenever the event $C$ (i.e. the collision along a continuous path) to be bounded in probability is constituted by the interaction of multiple events, joint chance constraints are used. At the foundation  of this framework lies the idea of approximating the probability of collision \cref{eq:totalprobabilityconfigurations} along a continuous path with the probability of the union of some discretized set of configurations $\mathbb{I}\subset \mathbb{S}$ (waypoints) along the path,
\begin{equation}
  P(C) \approxeq P\left(\bigvee_{i\in  \mathbb{I}} C_i \right).
  \label{eq:ccapprox}
\end{equation}
A bound on the right-hand side of \cref{eq:ccapprox} is called Joint Chance Constraint. Being Joint chance constraints computationally unapproachable, even in a discrete scenario, relaxation techniques have been developed. Most of these techniques rely on Boole's Lemma \cref{eq:booleslemma},
\begin{equation}
  \label{eq:booleslemma}
  P\left(\bigvee_{i\in  \mathbb{I}} C_i\right) \leq \sum_{i\in  \mathbb{I}} P(C_i),
\end{equation}
which over-bounds the probability of the union of the events with the sum of their individual probabilities. We refer to this method as \emph{Stage-Wise Approximation}.

Unfortunately, \cref{eq:ccapprox} does not tell how its two sides relate, meaning that the bound on $P\left(\bigvee_{i\in \mathbb{I}} C_i\right)$ could be invalid for $P(C)$. Indeed if the discretization chosen is too sparse, then the left-hand side of \cref{eq:booleslemma} will describe only a subset of the wanted event set, and the probability computed could be a lower approximation of the true value $P(C)$, as there is no accounting of the possibility of a collision happening between waypoints. Conversely, by choosing an excessively dense discretization, the gap between the bound and the true value of the probability grows, as shown in \cref{sfig:jointcc}.

In practice, this approach tends to be very conservative. When used for trajectory optimization, it rejects possibly feasible configurations \cite{freyCollisionProbabilitiesContinuousTime2020}.

%% file: RDsetup.tex
Following the analysis of collision probability along a continuous path in \Cref{sec:statement}, we investigate how the seemingly very different assumptions (H2) and (H3) relate.

Indeed consider the sensitivity to the parameter $T$, denoting the bounds of the variable $\theta$, of both \cref{eq:probabilityRVparam} and \cref{eq:volterraprob} with parametrization \cref{eq:tubeparametrizationexample}:
\begin{align}
     &S_r^{H2}(\mu_R,\overline{T}):=\frac{dP^a_{H2}}{dT}\Big |_{T=\overline{T}},\label{eq:rdH2} \\
     &S_r^{H3}(\mu_R,\overline{T}):=\frac{dP^a_{H3}}{dT}\Big |_{T=\overline{T}}.\label{eq:rdH3}
\end{align}
Notice the explicit dependence of \cref{eq:rdH2} and \cref{eq:rdH3} on the trajectory under scrutiny $\mu_R(s)$ and the dimensional parameter $T$, as their knowledge implies the parameterization \cref{eq:tubeparametrizationexample} .
\begin{lemma}
    The sensitivity under (H2) and under (H3) are equivalent in $T=0$, i.e.
    \begin{equation}
        \label{eq:lemmariskdensity}
       S_r^{H2}(\mu_R,0) = S_r^{H3}(\mu_R,0) %
    \end{equation}
    \label{lem:riskdensity}
\end{lemma}
\begin{proof}
    The risk sensitivity $S_r^{H2}(\mu_R,\overline{T})$ and $S_r^{H3}(\mu_R,\overline{T})$ can be computed by the Liebniz Integral rule as
    \begin{align}
    \label{eq:sensH3}
        S_r^{H3}(\mu_R,\overline{T}) =&\int_{0}^{1} \left( \mathcal{N}(\gamma,\nu) \Gamma(\gamma,\nu) \right )\Big |_{\nu=\overline{T}} d\gamma+ \nonumber \\  &\int_{0}^{1} \left(\mathcal{N}(\gamma,\nu) \Gamma(\gamma,\nu)\right)\Big |_{\nu=-\overline{T}} d\gamma,
    \end{align}
    \begin{align}
    \label{eq:sensH2}
        &S_r^{H2}(\mu_R,\overline{T}) = \nonumber \\&S_r^{H3}(\mu_R,\overline{T}) \exp \left(-\int_{0}^1 \int_{-\overline{T}}^{\overline{T}} \mathcal{N}(\gamma,\nu) \Gamma(\gamma,\nu) d\gamma d\nu\right),
    \end{align}
    where $\mathcal{N}(\gamma,\nu) := \mathcal{N}(\Phi(\gamma,\nu) - \mu_O|0,\Sigma_T)$.
    In $T=0$, $\Gamma(\gamma,0) = |\nicefrac{d\mu_R(s)}{d\gamma}|$, and the preceding integrals both evaluate to 
    \begin{equation*}
        2\int_{0}^{1} \mathcal{N}(\mu_R(\gamma)-\mu_O|0,\Sigma_T)\left|\frac{d \mu_R}{d \gamma}\right| d\gamma,
    \end{equation*}
    where $|\cdot|$ is the 2-norm.
\end{proof}
The result is a functional taking a path parametrization $\mu_R$ and giving a positive scalar as an output. 
We refer to this map as the \emph{Risk Density} along $\mu_R$, namely
\begin{equation}
    r_d(\mu_R(\cdot)) = 2\int_{0}^{1} \mathcal{N}(\mu_R(\gamma)-\mu_O|0,\Sigma_T)\left|\frac{d \mu_R}{d \gamma}\right| d\gamma.
    \label{eq:riskmeasure}
\end{equation}

Within the limits of the parametrization choice \cref{eq:tubeparametrizationexample},  \cref{lem:riskdensity} indicates that (H2) and (H3) converge in the limit to the same formulation. The sensitivity of the collision probability to small changes in the robot-obstacle dimensions is the same in apparently very different conditions, i.e. the case when only the initial position of the robot/obstacle is uncertain, or when the obstacle can appear randomly in front of the robot.

Besides the interesting theoretical result that the concept of \emph{Risk Density} reconciles two very far points in the range of possible assumptions about collision interaction, we argue that \cref{eq:riskmeasure} also has more practical uses. In the rest of this paper we will show how well an approximation in the form 
\begin{equation}
    \label{eq:taylorwrtT}
        P_a(C,T) \simeq rd(\mu_R(\cdot)) T,
\end{equation}
can estimate collision probabilities.

\begin{remark}
    Notice how our derivation assumes a 2D Euclidean workspace and limits the robot's transformations to translations. The assumption on the workspace dimensions is for simplicity sake and 3D formulations of \cref{eq:probabilityRVparam}, \cref{eq:volterraprob}, \cref{eq:lemmariskdensity} and \cref{eq:taylorwrtT}  are obtained in \cref{app:A}. We leave the analysis of the collision probability with rotations to future work.
\end{remark}
\subsection*{Risk Density for multiple Obstacles}
To compute the equivalent of \cref{eq:riskmeasure} for an environment where multiple obstacles $O_i$ live, we first introduce an assumption on the interaction between the collision events w.r.t. different obstacles.
Namely we assume
\begin{equation}
    \label{eq:independentobsta}
    P(C|O_i) \perp\!\!\!\perp P(C|O_j), \forall \  i \not= j,
\end{equation}
i.e. that whether the robot collides with an obstacle or not does not affect the probability of collision with another obstacle.
\Cref{eq:independentobsta} let us write the total probability of collision w.r.t. a set of obstacles $\mathbb{O}$ as\footnote{We derive here the expression for multiple obstacles under \textbf{(H2)}, but the procedure remains unvaried under \textbf{(H2)}.} 
\small
\begin{align}
    \label{eq:multipleobsta}
    &P(C)= 1- \prod_{O_i\in \mathbb{O}} \left( 1- P(C|O_i) \right) \nonumber \\
    &= 1- \prod_{i\in \text{ind}(\mathbb{O})} \left( 1 - \left( 1 - \exp  \left( - \int_{-T_i}^{T_i} \int_0^1 \mathcal{N}_i(\gamma,\nu) \Gamma_i(\gamma,\nu) \gamma \nu  \right) \right) \right) \nonumber \\
    &= 1- \prod_{i\in \text{ind}(\mathbb{O})} \exp  \left( - \int_{-T_i}^{T_i} \int_0^1 \mathcal{N}_i(\gamma,\nu) \Gamma_i(\gamma,\nu) \gamma \nu  \right)
\end{align}
\normalsize
where $T_i$ are the radii of $O_i$, while $\mathcal{N}_i(\gamma,\nu) = \mathcal{N}(\Phi(\gamma,\nu) - \mu_{Oi}|0,\Sigma_{Ti})$ and $\Gamma_i(\gamma,\nu) = |\det (\nabla \mathcal{N}_i)| $. Define the vector $\mathbb{T} = [T_1,\dots, T_i,\dots]^\top$ and the shorthand
\begin{equation}
    \label{eq:shorthand}
    \overline{P}_i := \exp  \left( - \int_{-T_i}^{T_i} \int_0^1 \mathcal{N}_i(\gamma,\nu) \Gamma_i(\gamma,\nu) \gamma \nu  \right).
\end{equation}
Taking the sensitivity of \cref{eq:multipleobsta} w.r.t. $\mathbb{T}$ is trivial by using the derivative product rule
\begin{align}
    \label{eq:productrule}
    &\nabla_{\mathbb{T}} (\ref{eq:multipleobsta}) = - \sum_{i\in \text{ind}(\mathbb{O})}\left( \left( \prod_{j\in \text{ind}(\mathbb{O}), j\ne i}  \overline{P}_j \right) \nabla_{\mathbb{T}}\overline{P}_i \right) \nonumber \\
    &= - \sum_{i\in \text{ind}(\mathbb{O})}\left( \left( \prod_{j\in \text{ind}(\mathbb{O}), j\ne i}  \overline{P}_j \right) \left[ 0, \dots 0, \frac{d\overline{P}_i}{dT_i}, 0, \dots \right]^\top  \right).
\end{align}
Evaluating \cref{eq:productrule} in $\mathbb{T} = \overline{0}$, as each $\overline{P}_j(T_j=0) = 1$ and $\frac{d\overline{P}_i}{dT_i} = -rd_{O_i}(\mu_R)$, results in the vector
\begin{equation}
    \label{eq:multipleRD}
    rd_{\mathbb{O}}(\mu_R) = 
    \begin{bmatrix}
    \vdots\\
    rd_{O_i}(\mu_R)\\
    \vdots
    \end{bmatrix}.
\end{equation}
Moreover we can recover the familiar approximation \cref{eq:taylorwrtT} as 
\begin{equation}
\label{eq:rdmultipleproof}
    P_a(C,\mathbb{T}) = \sum_{i \in \text{ind}(\mathbb{O})} rd_{O_i}(\mu_R) T_i = rd_{\mathbb{O}}(\mu_R)^\top \mathbb{T}.
\end{equation}

%% file: RDdiscussion.tex
\subsection*{Setup}
\label{subsec:setup}
To analyze the proprieties of the proposed function, we introduce a simple case study, compatible with the assumptions of \Cref{sec:back} and (H3), which, despite its simplicity, captures some of the problem richness.

A robot vehicle shaped as a disk of radius $r_r=0.05$, must navigate a room with a flat floor with a round obstacle of radius $r_O=0.05$, roughly located in the middle of the room.
Both the robot initial position and the obstacle location are known with some uncertainty that we can model as two independent Gaussian distributions $N_R\sim\mathcal{N}(0,\Sigma_R)$, $N_O\sim\mathcal{N}(0,\Sigma_O)$.
We want to evaluate the collision probability along three different paths $\mu_A,\ \mu_B,$ and $\mu_C$, all starting from the origin $x_0$ and finishing $x_F=\left[\begin{smallmatrix}5\\0\end{smallmatrix}\right]$, while the nominal position of the obstacle is $\mu_O = \left[\begin{smallmatrix}\nicefrac{5}{2}\\0\end{smallmatrix}\right]$. The first trajectory, $\mu_A$, is a straight line crossing the nominal position of the obstacle, $\mu_B$ and $\mu_C$ are parabolas of which the former is designed to graze the obstacle in the nominal position, while the latter passes close to it albeit always staying further than $r$.
\begin{figure}
    \centering
    \includegraphics[trim={3.5cm 6cm 3.5cm 5cm},clip,width=0.95\linewidth]{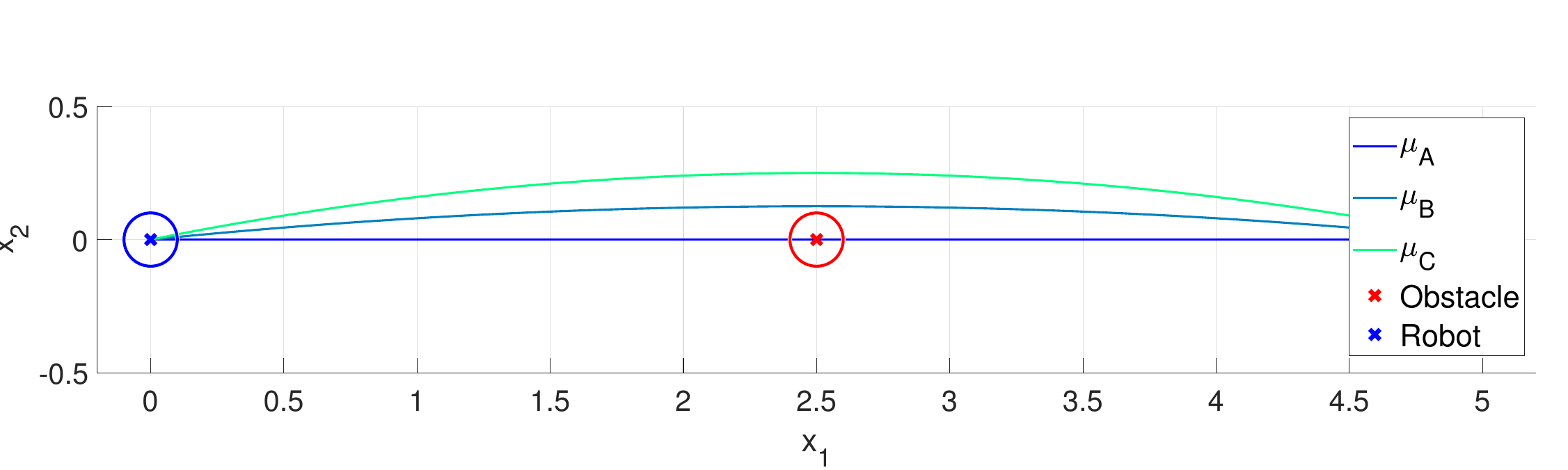}
    \caption{Case in study: nominal case. Close-up of the set-up with the three parabolic trajectories $\mu_A,\mu_B,\mu_C$. The trajectories are parametrized as $\left\{ \left[\begin{smallmatrix} s\\0 \end{smallmatrix}\right], \left[\begin{smallmatrix} s\\-\nicefrac{s^2-s}{2} \end{smallmatrix}\right], \left[\begin{smallmatrix} s\\ s -s^2 \end{smallmatrix}\right]  \right\}$. The initial condition $x_0$ is in the origin, while $x_F = \left[\begin{smallmatrix} 5\\0 \end{smallmatrix}\right]$. The robot and the obstacle are disks with radius $r_{r/O}=0.05$.  \label{fig:threetrajectories}}
\end{figure}
We look at how the probability value, computed with different approaches, varies with regard to the magnitude of the covariance matrix $\Sigma_T$.
To look at the variation of this parameter we multiply a identity matrix by a scalar value $\sigma$ so that $\Sigma_T = I \cdot \sigma$; this parameter takes $10$ values spaced on a logarithmic scale in the range $(10^{-3},1)$.
We emphasize that $\Sigma_T$ is a simulation parameter fixing the combined robot-obstacle uncertainty. The choice of its value, in the context of this experiment, is arbitrary and imposes an isotropic uncertainty.

\subsubsection*{Ground Truth}

To compute what we take as the ground truth, we perform a Monte Carlo simulation, consisting of $10^4$ trials, for each of the $30$ cases in scrutiny. At the start of each trial a random sample $\mu_0^*$ is taken from $\mathcal{N}(x|\mu_O,\Sigma_T)$\footnote{via the MATLAB function mvrnd()}, representing the actual position of the obstacle for the simulation. Afterwards, given a discretization resolution $ds$ for the curvilinear parameter $s$, at each step if $|\mu_i(s) - \mu_O^*| \leq r_r+r_o$ then the simulation has stopped and a collision is registered. The steps of this computation are fully enumerated in \Cref{app:A} (\cref{alg:groundtruth}), where we also include procedures for other assumptions. From the Monte Carlo simulation results we compute the probability ground-truth $P_T(C)$ for each of the inquired scenarios.

We compute the Monte Carlo simulation for $\mu_B$  under H1 H2 H3 to make their differences apparent. In this experiment the covariance parameter $\sigma$ lives in a logarithmic range $\left[10^{-3}, \ 10^{0}\right]$ and the radius $r$ of the combined object is varied linearly in $\left[ 0.1, \ 1\right]$.

\subsubsection{Naive Set Parametrization Probability Estimation}
The approximation given by \cref{eq:probabilityRVparam}, using the parametrization (\ref{eq:tubeparametrizationexample}), is computed for all the considered trajectories.

\subsubsection{Grid-Based Probability Estimation}

We evaluate the approximation given by the multiplicative formula \cref{eq:gridprobbelief} in various scenarios.
The set of cells $\mathbf{m_c}:=\{ m_1,\dots,m_N \}$ which the robot intersects is obtained by rasterizing $D_{T}$ on the grid.

As it was observed, the belief value \cref{eq:bayesbelief} associated to each cell must be corrected to take the size of the cell into account. We can do so by letting
\begin{equation}
    \label{eq:cellprob}
    P(C_{m_i}) = \int_{\mathcal{D}(m_i)} \mathcal{N}(x|\mu_O,\Sigma_T)dx,
\end{equation}
where $\mathcal{D}(x)$ gives the dominion of the cell. This formula obviously accounts for the area of the tiles.
Then \cref{eq:gridprobbelief} is rewritten as
\begin{equation}
    \label{eq:gridprobtot}
    P_g(C) = 1 - \prod_{m_i\in \mathbf{m_c}} (1-P(C_{m_i}))
\end{equation}

In addition to the variation of the covariance matrix and the different trajectories, we also consider different grid resolutions in our benchmark.

\subsubsection{Stage-wise Probability Estimation}

Combining the two approximations used by Chance Constraints, the \emph{Stage-wise approximation} amounts to
\begin{equation}
    P_{cc}(C) \approx \sum_{i=1}^N P(C_i)
    \label{eq:stagewise}
\end{equation}
in which each $P(C_i)$ is calculated as in \cref{eq:ccapprox}.
Also, being $P_{cc}(C)$ a probability, we consider the performance of the approximation by saturating it to $1$.
We also vary the number of waypoints $N$ linearly from $25$ to $300$.

\subsubsection*{Risk Density Based Probability Approximation}

The approximation we employ is \cref{eq:taylorwrtT} setting $T=r$, namely
\begin{equation}
    \label{eq:pathintrinsicapprox}
    P_r(C) = rd(\mu_R(s)) \cdot r,
\end{equation}
where $r = r_r+r_o$ is the combined radius.
As it was the case for \cref{eq:stagewise}, we saturate the value of \cref{eq:pathintrinsicapprox} to $1$.

\subsection*{Assumptions Comparison}

The results of the Monte Carlo simulations for the considered assumptions is shown here. The algorithms are detailed in \cref{app:A}, the number of trials is set to $N_t = 10^4$, the number of partition of the parameter space $[0, 1]$ is set to $N_s = 10^4$.

\begin{figure}[]
    \centering
    \includegraphics[trim = 0 0cm 0 0,width=0.85\linewidth]{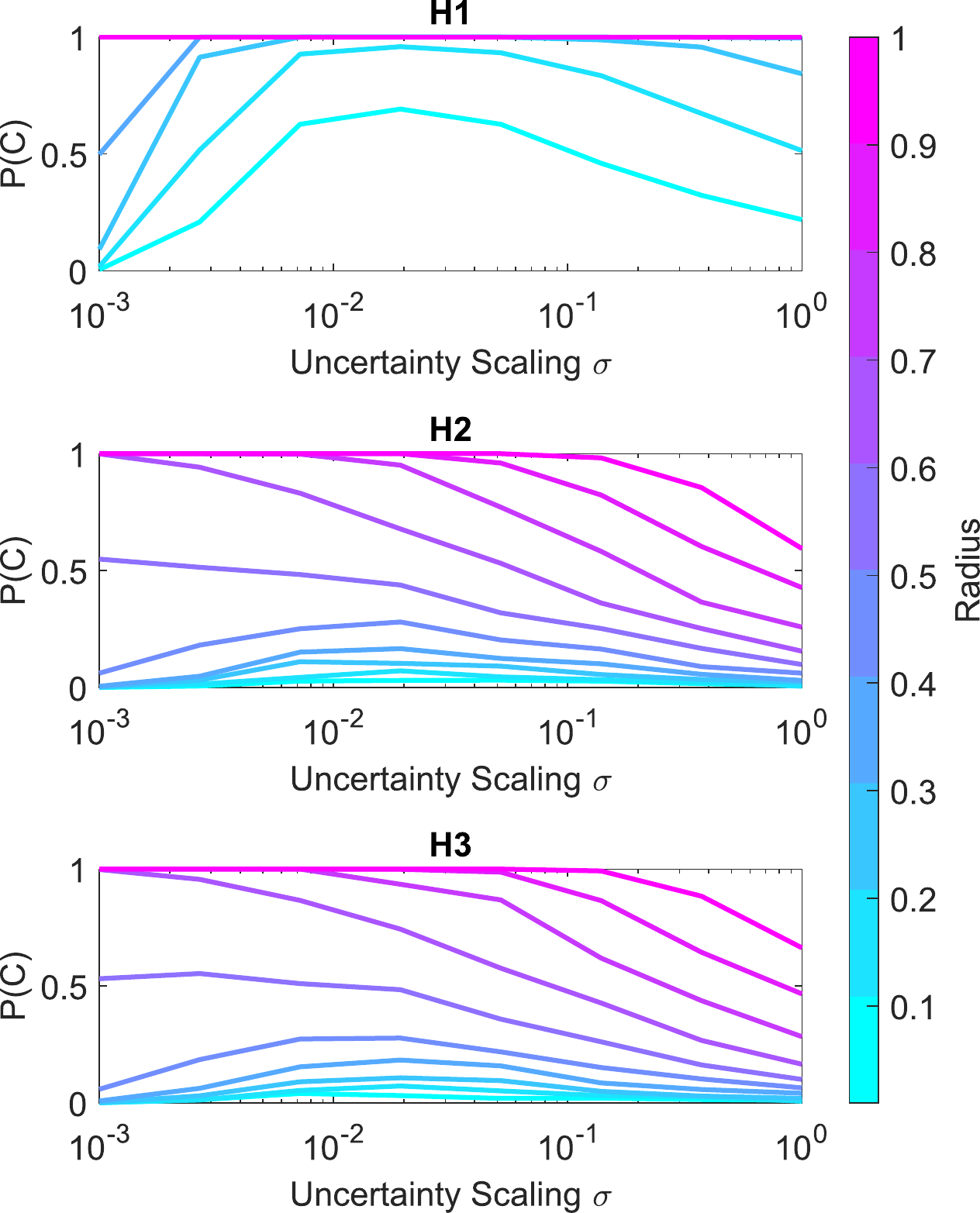}
    \caption{Comparison of Collision Event assumptions. Each plot shows the result of the Monte Carlo probability estimate under different assumptions. The color gradient indicates different radiuses. \label{fig:H1vsH2vsH3}}
\end{figure}
In \Cref{fig:H1vsH2vsH3}, the probabilities given by the trials are shown. We immediately see how the simulation under H1 results in probabilities tending to $1$ even for very small radii and high uncertainty values. This shows how H1, due to the fine discretization, is not fit to model this continuous setting.

In H2 and H3 case, while differences are present, the results are comparable. As a consequence, while the following subsection display results simulated under (H3), \cref{eq:taylorwrtT} could be used also in (H2) case.

\subsection*{Comparison Metric}

Denote $M_T\in [0,1]^{3 \times 10}$ as the matrix containing the ground truth for all the scenarios, and $M_P \in \mathbb{R}_+^{3 \times 10}$ as the matrix containing the same values computed by any of the presented approaches.

The error matrix is defined as
\begin{equation}
    \label{eq:errormatrix}
    M_e = M_T - M_P.
\end{equation}
The Frobenius Norm $||\cdot||_F$ and the maximum element in absolute value $\max(|\cdot|)$ are used to compare  the methods outlined in \Cref{subsec:setup}.

\subsection*{Results}
\label{subsec:Results}
\begin{figure}
    \centering
    \includegraphics[trim =3.5cm 0cm 3cm 1.6cm,clip,width=0.9\linewidth]{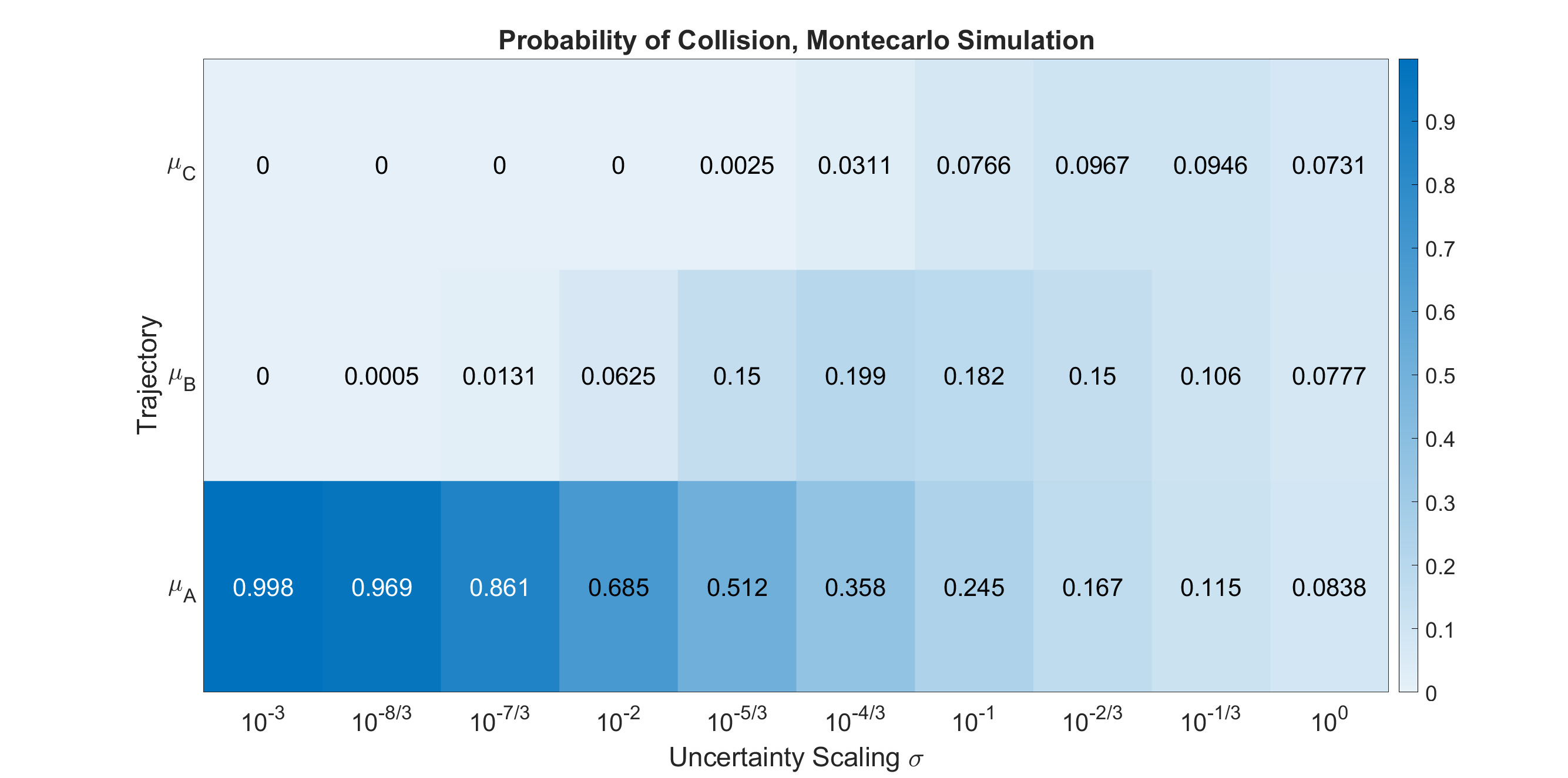}
    \caption{Probability of Collision, Monte Carlo Simulations.
        The heatmap shows the probability of collision along the trajectories $\mu_A$, $\mu_B$, and $\mu_C$, as a function of increasing values of the position error covariance $\Sigma_T$. Each cell is the result of the ensemble of the Monte Carlo Trials for that combination of covariance matrix and trajectory. This heatmap is considered the ground truth.\label{fig:montecarlo}}
\end{figure}
\begin{figure}
    \centering
    \includegraphics[trim =0cm 0cm 0cm 1.2cm,clip,width=\linewidth]{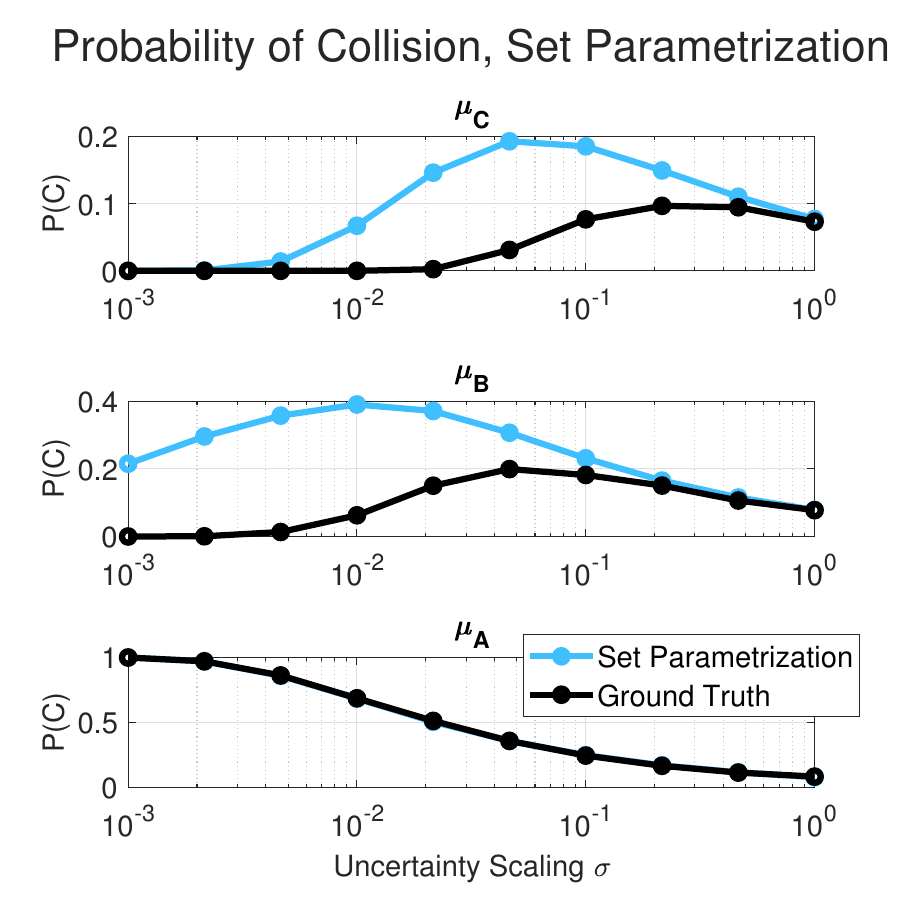}
    \caption{Probability of Collision, Set Parametrization. Values result of the integral \cref{eq:probabilityRVparam}. We see how, while the parametrization is accurate at evaluating the probability of collision of $\mu_A$, the approximation starts to falter as soon as the curvature of the trajectory is not null.  \label{fig:oneintegral}}
\end{figure}
\begin{figure*}
    \centering
    \includegraphics[trim =0cm 0cm 0cm 1.5cm,clip,width=0.9\textwidth]{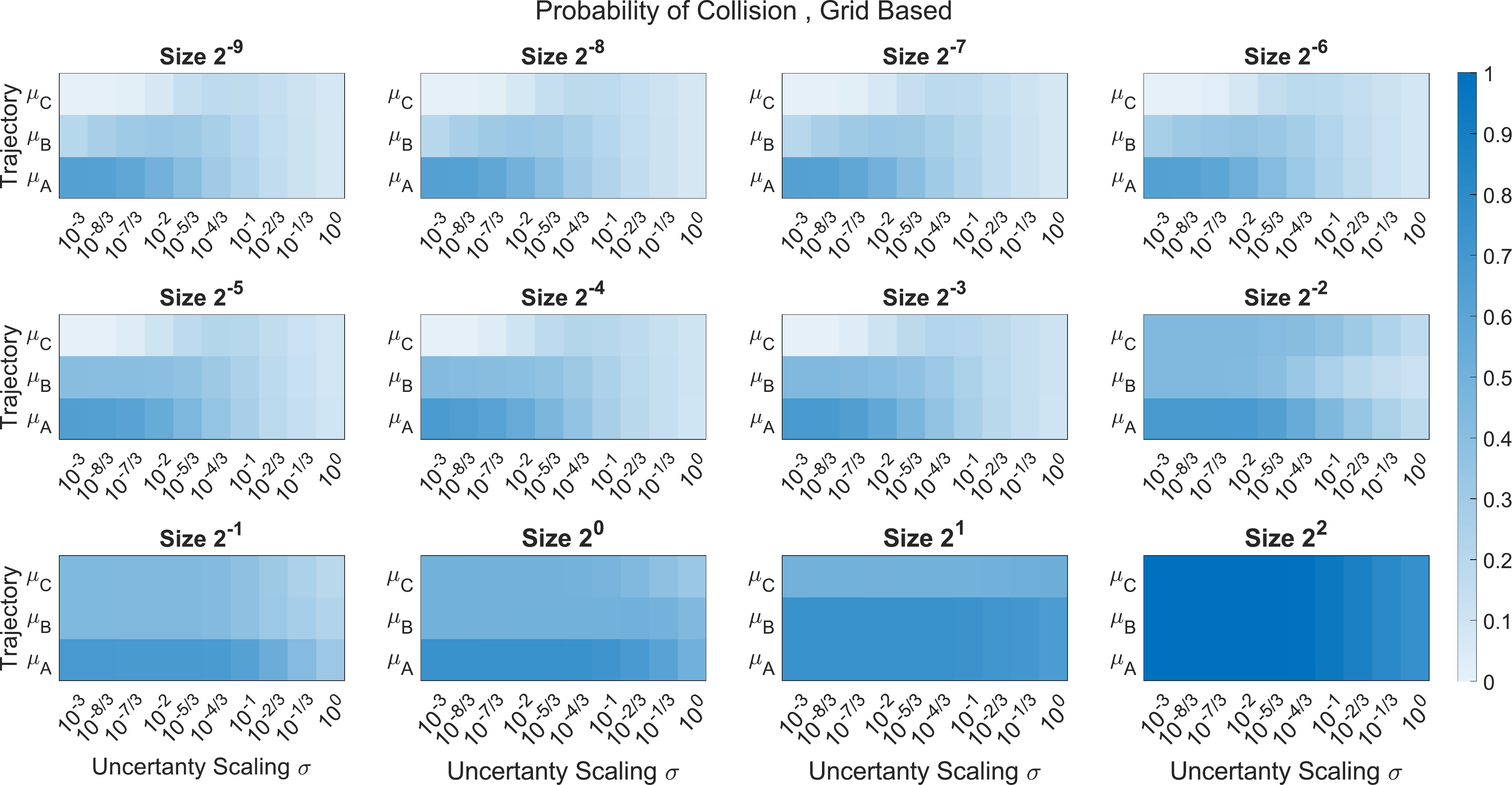}
    \caption{Probability of Collision, Grid based method. The heatmaps show the value of the probability of collision with regard of the size of the square tiling. The probability value associated to each pairing of trajectory and $\sigma$ converges as the resolution of the grid gets finer.  \label{fig:gridheatmaps}}
\end{figure*}
\begin{figure}
    \centering
    \includegraphics[trim =0cm 0cm 0cm 1.2cm,clip,width=0.9\linewidth]{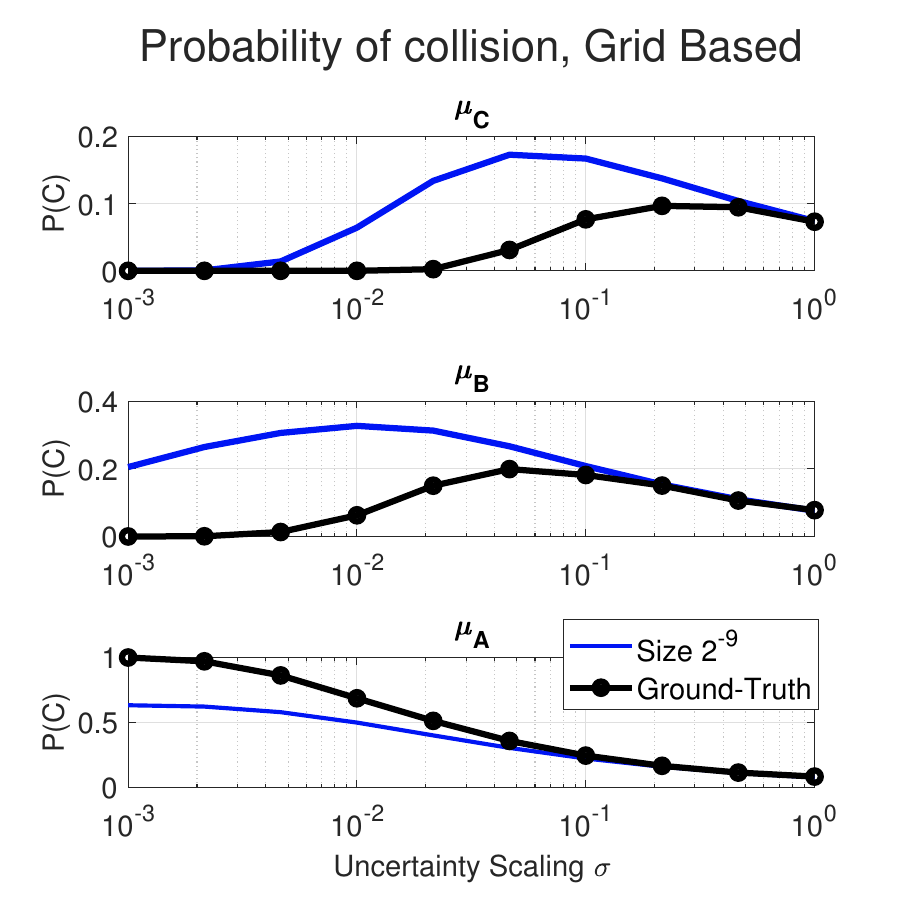}
    \caption{Probability of Collision, Grid based method - tile size $2^{-9}$. As the resolution gets finer, the grid-based approach converges to a value, which however is not the amount result of Monte Carlo simulation. \label{fig:allgrid}}
\end{figure}
\begin{figure*}
    \centering
    \includegraphics[trim =0cm 0cm 0cm 0.7cm,clip,width=0.9\textwidth]{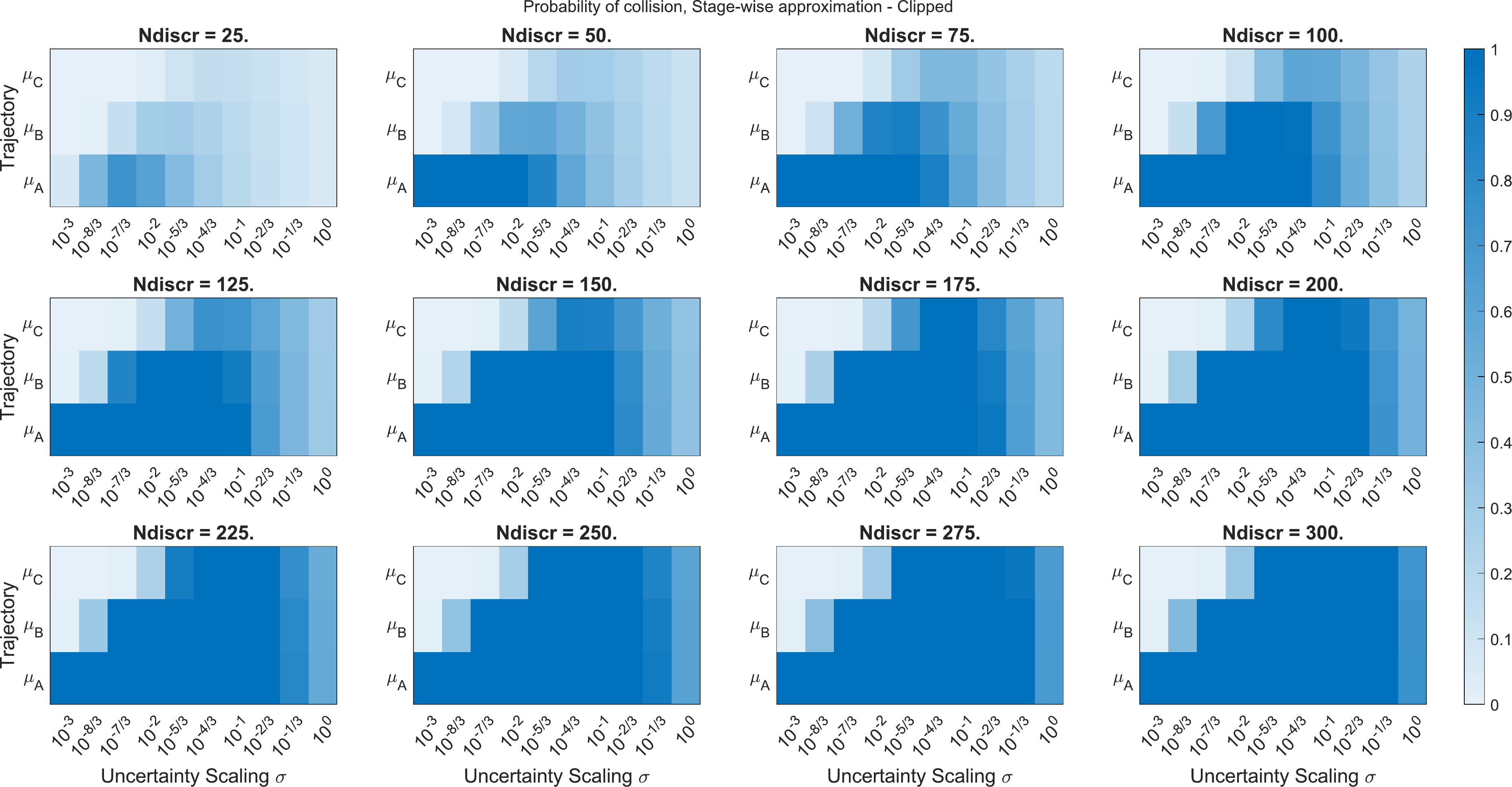}
    \caption{Probability of Collision, Stage-wise method. Result of the approximation using \cref{eq:stagewise}. The  number of waypoints, in which the approximation is computed, is varied. The values computed are ceiled at $1$. We see how the approximation is not informative after a certain number of waypoints as it always predicts a collision even in the safest of cases. \label{fig:booleclipped}}
\end{figure*}
Monte Carlo simulation results are displayed in \Cref{fig:montecarlo} for all scenarios. The heatmap shows both the trend with respect to $\sigma$ and the trajectory choice. While $\mu_A$ is clearly the riskier trajectory and presents a peak in collision probability for very low values of $\sigma$, we see how the probability of collision shows a maximum with respect to the covariance in each path. At higher values of covariance, probability between the trajectories tends to coincide.

The probability approximated by the set parametrization approach \cref{eq:probabilityRVparam} is very accurate for the first trajectory, as shown in \cref{fig:oneintegral}, but quickly degrades and overbounds as the trajectory gets further away and the curvature increases.

The grid-based approximation \cref{eq:gridprobtot} estimations are shown in \cref{fig:gridheatmaps}. In each of the heatmaps, the probability value is shown regarding uncertainty and the path picked. In \Cref{fig:allgrid} the Monte Carlo simulation is shown along with the values obtained from the grid-based approach with cell size $2^{-9}$. It is apparent that the estimation converges to a specific curve by considering a finer and finer grid, but this value approximates the ground truth poorly.

The heatmaps in \Cref{fig:booleclipped} present the behaviour of the stage-wise approximation \cref{eq:stagewise}.
The method tends to be more and more conservative as the number of waypoints $N$ increases. Moreover, saturation leads to a loss of information about the trajectory, as the estimation predicts a collision regardless of the trajectory. The value of the approximation can either be smaller or larger than the actual probability depending on the number of waypoints, as shown in \Cref{fig:ccvsmonte}.
\begin{figure}
    \centering
    \includegraphics[trim =0cm 2cm 0cm 1.8cm,clip,width=0.9\linewidth]{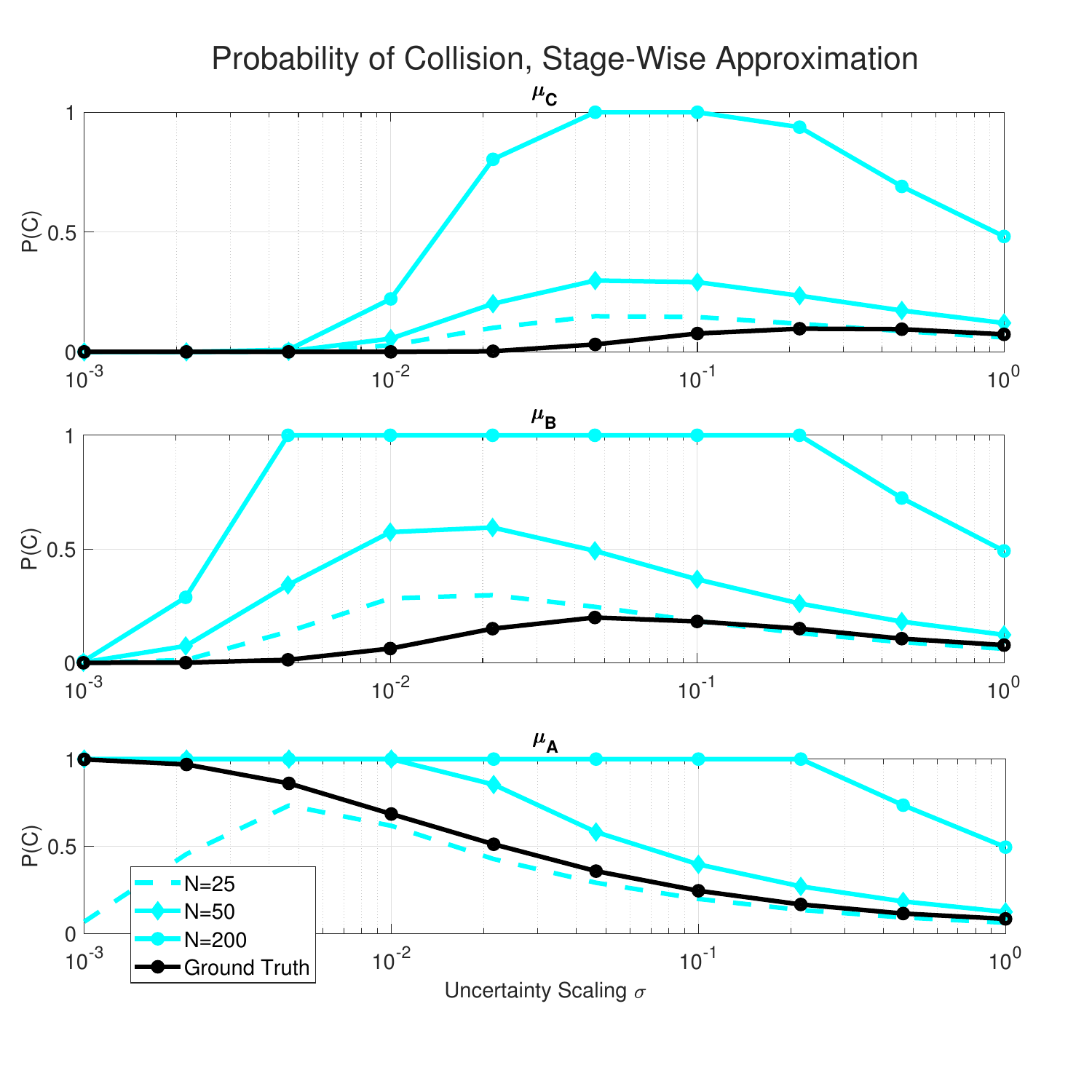}
    \caption{Probability approximation, Stage-wise approximation. Values result of \cref{eq:stagewise}, for some selected number of waypoints $N$, compared against Monte Carlo simulation. While usually this approach overbounds, notice how when $N=25$ the probability of $\mu_A$ is lower than the actual value.
        \label{fig:ccvsmonte}}
\end{figure}
\begin{figure}
    \centering
    \includegraphics[trim =0cm 0cm 0cm 0.75cm,clip,width=0.9\linewidth]{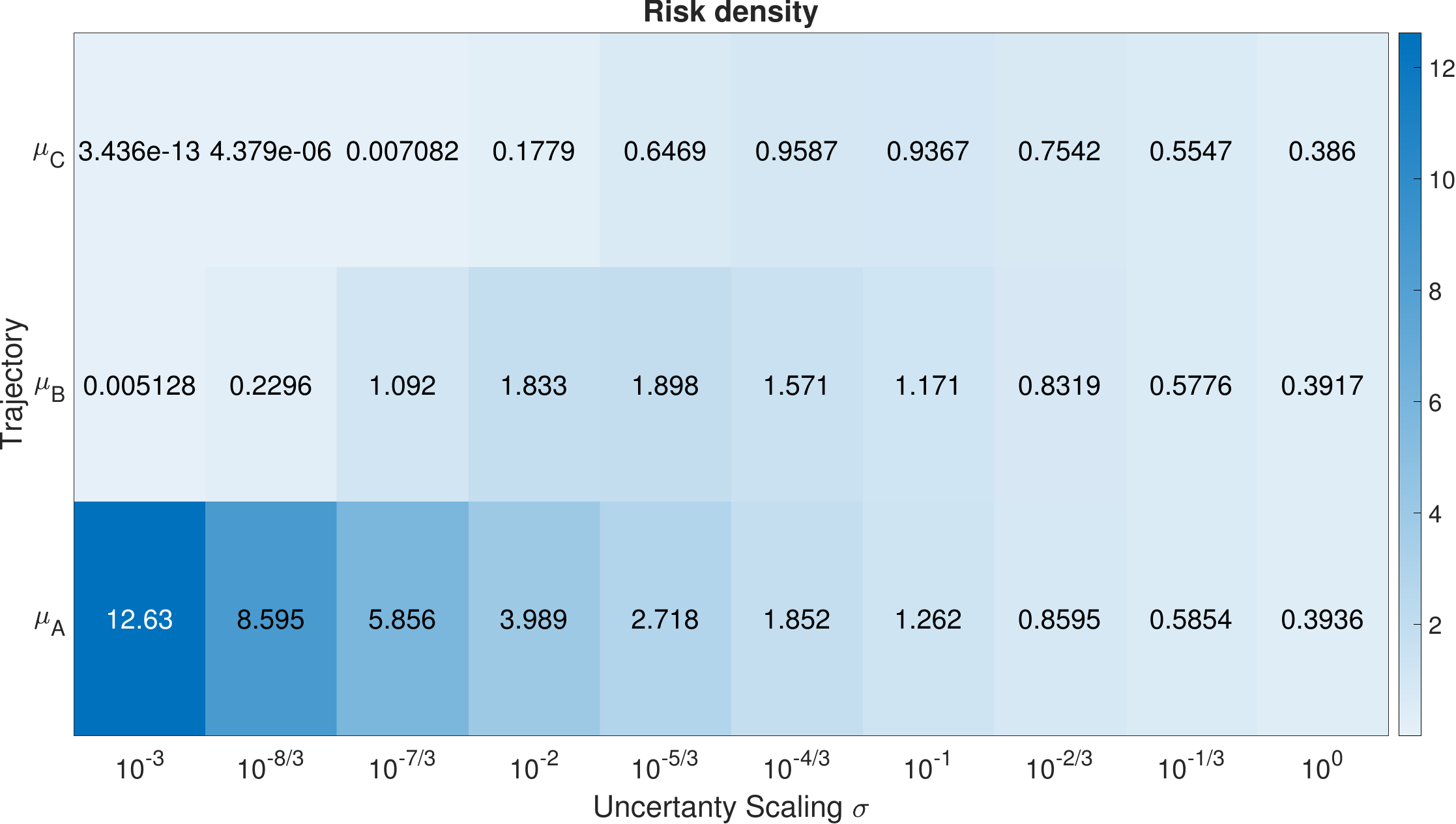}
    \caption{Risk Density. Behaviour of \cref{eq:riskmeasure} w.r.t. uncertainty scaling and path chosen. While the vales shown are not probabilities, the behavior with respect to $\sigma$ is maintained. \label{fig:riskdensity}}
\end{figure}
\begin{figure}
    \centering
    \includegraphics[trim =0cm 0cm 0cm 1.5cm,clip,width=0.9\linewidth]{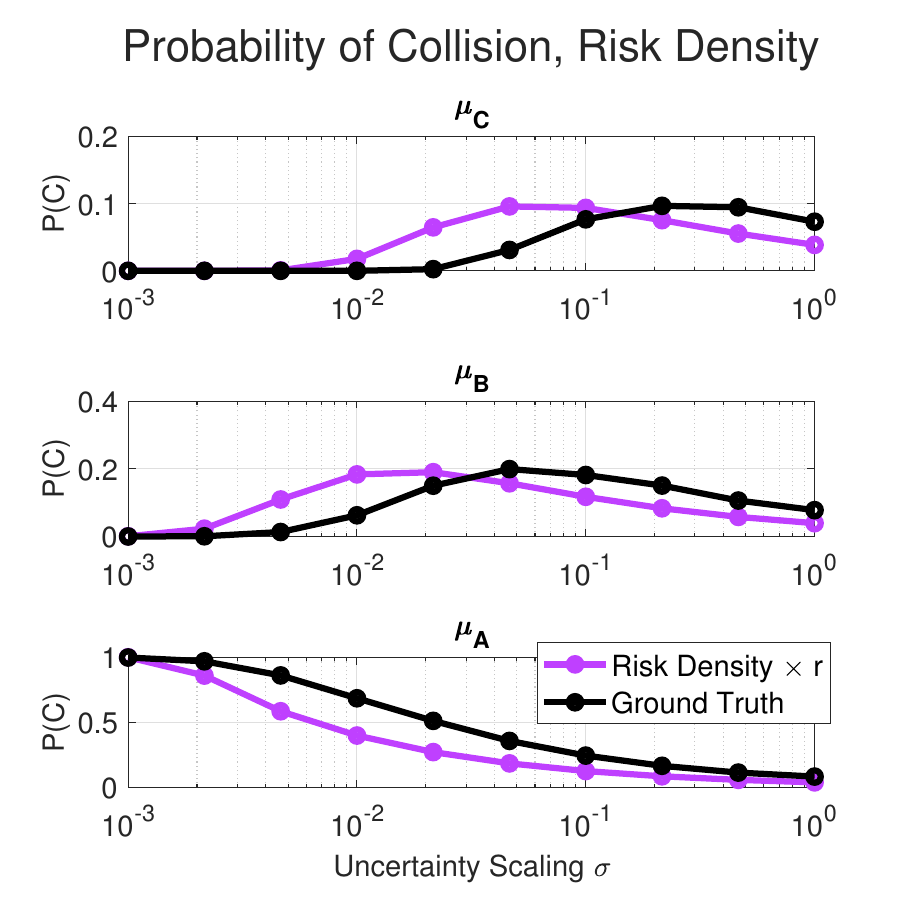}
    \caption{Probability of Collision, Risk Density. Behaviour of the proposed approximation \cref{eq:riskmeasure} against Monte Carlo simulation. Notice how the purple lines have the same \enquote{shape} of the heatmap shown in \cref{fig:riskdensity}, as these two graphs differ only by the multiplication factor $r$, as in \cref{eq:pathintrinsicapprox}. \label{fig:riskvsmonte}}
\end{figure}

The trend of \emph{risk density} \cref{eq:riskmeasure} considered in isolation is displayed in \cref{fig:riskdensity}. The behaviour of the approximation \cref{eq:pathintrinsicapprox} is compared to the ground truth given by Monte Carlo in \cref{fig:riskvsmonte}.

A summarizing view is given in \cref{fig:errorsall}. Here the errors against the ground truth are plotted with regard to $\sigma$.
The set parametrization approach \cref{eq:probabilityRVparam}, shown in light blue in \cref{fig:errorsall}, is the most accurate in appraising the risk of $\mu_A$, while the discrepancy with the Monte Carlo ground truth is larger the further the path is from the obstacle.

A more detailed comparison is given in \cref{tb:normandtime}. The table shows how our approach gives a $12\%$ improvement in the Frobenius norm of the estimation error matrix compared to the stage-wise approximation. The unsaturated values are not included, as is evident how the stage-wise approximation greatly overestimates the actual probability whenever the path in inquiry is \enquote{closer} to the obstacle ($\mu_A$).

\begin{figure*}[]
    \centering
    \includegraphics[trim =3cm 0cm 3cm 0cm,clip,width=0.9\textwidth]{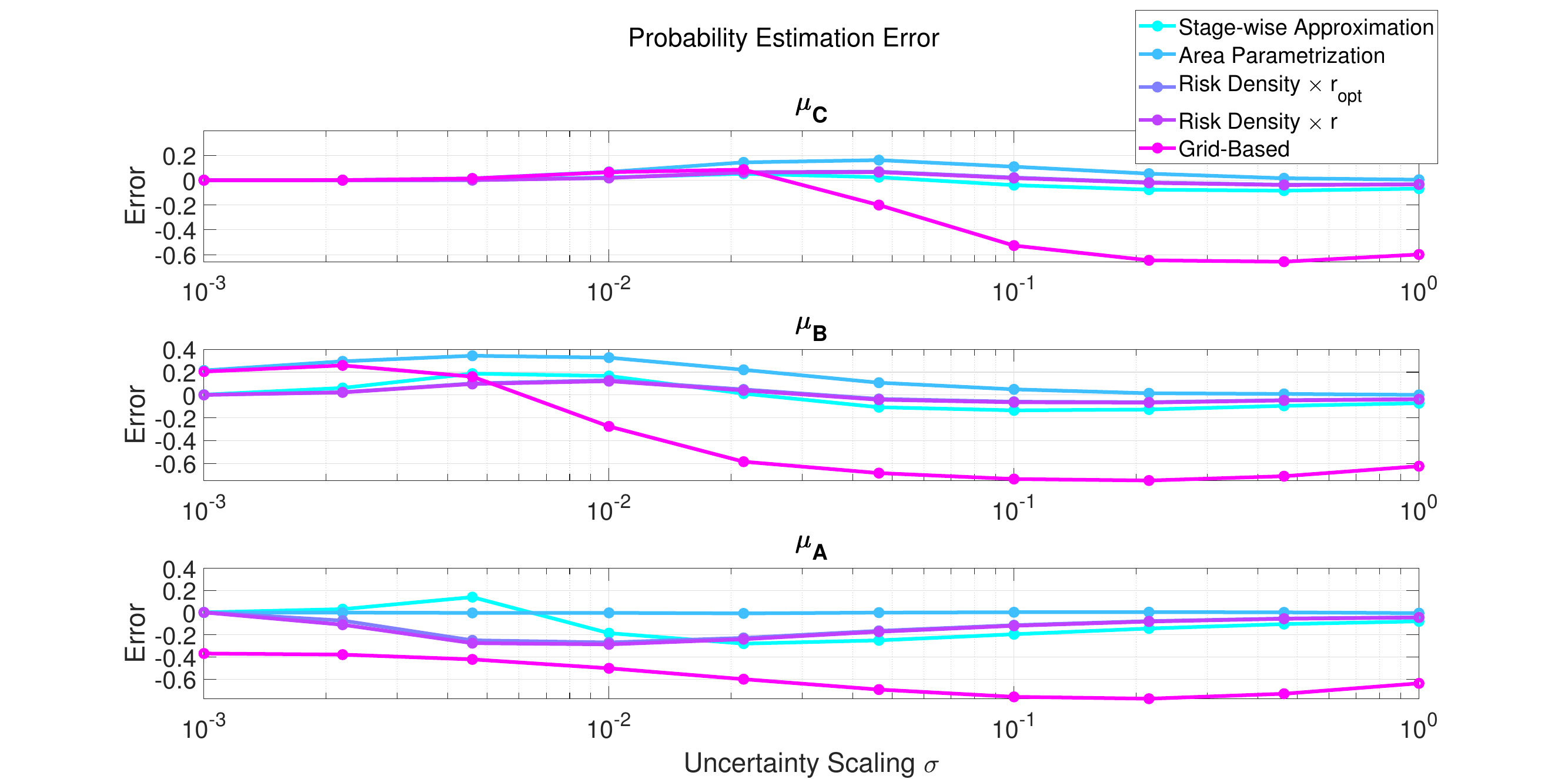}
    \caption{Errors against the Monte Carlo simulation. The error is given in the $[-1,1]$ range: if the value plotted is positive, there is an overestimation of the probability, conversely, a negative value denotes an underestimation. The stage-wise approximation error shown is computed while fixing $N$, the number of waypoints, to $50$. In the chart we show both the probability estimated by using as a factor $r$ and $r_{opt}$.\label{fig:errorsall}}
\end{figure*}
\begin{table*}[]
    \centering
    {\scriptsize
        \begin{tabular}{|l|c|c|cc|c|c|c}
            \hline          & Stage-wise {\tiny$(N=50)$}  & Parametrization \cref{eq:tubeparametrizationexample} & Risk density $r$              & $r_{opt}$({\tiny$=0.1043$}) & Grid-based (Size{\tiny $=2^{-9}$}) & Monte Carlo ({\tiny$N_T=10^4$}) \\ \hline
            $\max(|M_e|)$   & $0.2796$                    & $0.3438$                                             & $0.2863$                      & $0.2689$                    & $0.7754$                           & 0                               \\
            $||M_e||_F$     & $0.649$                     & $0.698$                                              & $0.579$                       & $0.545$                     & $2.869$                            & 0                               \\
            Comp Time $[s]$ & $(7.2\pm 2.2)\cdot 10^{-3}$ & $(8.7 \pm 0.5 ) \cdot 10^{-1}$                       & $(1.3 \pm 0.2) \cdot 10^{-2}$ & $\sim$                      & $380.43^*$                         & $29 \pm 3$                      \\ \hline
        \end{tabular}
    }
    \caption{Value of the matrix norm of the ensemble of errors $M_P$ and computation time of considered methods. The computation time refers to the execution of the whole set of $30$ scenarios. The stage-wise and risk density-based approach errors are from the saturated approximation. The values in the table were obtained by running each method $100$ times and computing the mean time and its standard deviation. All the simulations were run on a modern laptop with Matlab 2022b on an Intel i7-12700H with 32GB of Ram. \scriptsize{$^*$We include the computation time for the grid-based computation for completeness, but our simulation introduces complexity that would not be present in a standard implementation.}  \label{tb:normandtime}}
\end{table*}

%% file: RDdiscussion2.tex
As shown by the simulations, the accuracy of the approaches varies greatly, even in the simple scenario considered above.

All reviewed methods capture the qualitative behavior shown by the Montecarlo simulation, exhibiting a maximum in the probability estimate with respect to the uncertainty scaling for all approaches.
However, the parametrized probability integral \cref{eq:probabilityRVparam} is not a reliable approximation even for modest curvature values: this happens due to the fact that the parametrization \cref{eq:tubeparametrizationexample} of the integration set is not an homeomorphic map between $[0,1]\times[-T,T]$ and $\mathbb{W}$, leading to an overestimation of the actual probability value. Moreover, the computation is $1-2$ orders of magnitude slower with respect to other methods.

As mentioned in \Cref{sec:numval}, the ground value peak appears at very low values of $\sigma$ in $\mu_A$'s case. Ideally the peak would be at $\sigma = 0$, the deterministic scenario, since the probability of collision of a nominal path passing through an obstacle whose position is known with $100\%$ certainty. As the path moves away from the obstacle nominal position, the peaks also shift towards higher values: while in the deterministic case the robot would narrowly avoid the obstacle, a higher variance $\Sigma_T$ favours the collision event. 

When the resolution of the cells gets finer, the grid-based method does not converge to the probability predicted by the Montecarlo simulation. This fact is mainly due to the difference between the assumption underlying the grid approach and (H3): the latter assumes the independence of cells, which does not apply to this scenario.

The stage-wise approximation \cref{eq:stagewise} of the probability of collision is fast and, as such, suited to optimization. However, while \cref{eq:dutoitapprox} is not guaranteed to overbound the probability, we observe that for almost every number of waypoints, the result of \cref{eq:stagewise} is a substantial overapproximation of the ground truth. This conservativeness in an optimization setting could lead to categorizing admissible states as unfeasible.

At last, the proposed function \cref{eq:riskmeasure}, coupled with the approximation \cref{eq:pathintrinsicapprox}, shows some interesting properties. As no exogenous parameters are included in its computation, we call it \emph{path intrinsic}; namely, the probability \cref{eq:pathintrinsicapprox} is just a functional of the path $\mu_R(\cdot)$ taken by the robot and the combined radius $r$, while no other tuning knob, such as $N$, the number of waypoints, or the resolution of a grid, is considered. It could be argued that the $r$ parameter by which the density \cref{eq:riskmeasure} is scaled in \cref{eq:pathintrinsicapprox} is not the optimal one to estimate the probability of collision. To validate our choice, we obtain $r_{opt}$ by minimizing $||M_e||_F$ for an arbitrary choice of the scaling parameter. As shown in \cref{tb:normandtime}, the variation in the values of the error is modest and does not justify the need for this optimization.

Curiously, the approximation  \cref{eq:pathintrinsicapprox} is more accurate than the integral \cref{eq:probabilityRVparam}: taking the derivative of \cref{eq:probabilityRVparam} at \enquote{thickness} $0$ rids \cref{eq:pathintrinsicapprox} of the local homeomorphicity issue encountered in \cref{eq:probabilityRVparam}. Furthermore, the method also seems to be competitive in regard to the accuracy of the approximation, being the least conservative among the considered approaches.

While limited in scope, the experiment also shows that the computational cost of the Risk Density approach is of the same order of magnitude ($\sim\times2$) as the fastest method (Stage-wise approximation), as shown in \cref{tb:normandtime}. This relatively small computational complexity is explained by the fact that our approximation \cref{eq:pathintrinsicapprox} includes only a uni-dimensional integral, which is easily evaluated numerically.

%% file: RDnewapprox.tex
As shown in \cref{sec:numval}, risk density can be used to give an approximation of the collision probability when the position of the robot and the obstacle are uncertain.
We now depict another situation in which \cref{eq:riskmeasure} is useful.
Let us assume that we have already computed the probability of collision very accurately (e.g. through a long running MC simulation) for an environment we know the dimensions of very well. We call this estimate of the collision probability $P_M^{(i)}$ and its associated dimensional parameter $T_i$. Assume now that the dimension of the robot and/or the obstacle, i.e. the environment, changes, while the trajectory followed by the robot, $\mu_R$, is unvaried. The previously computed estimate of the collision probability is now invalid, but computing the collision probability $P_M^{(i+1)}$ would imply a hefty computational cost. To ameliorate this issue, we can use the risk density based approximation \cref{eq:taylorwrtT}.

A first order approximation of $P_M^{(i+1)}$ given by the sensitivity is
\begin{equation}
  \label{eq:sensitivityMCapprox}
  P_M^{(i+1)} \approxeq P_M^{(i)} + S_r^{Hi}(\mu_R,T_i)\cdot d_{T_i},
\end{equation}
where $Hi$ stands for either $H2$ or $H3$. We can obtain a similar relation by exploiting the collision probability approximation given by the risk density.
Explicitly, when both the probabilities $P_M^{(i+1)}$ and $P_M^{(i)}$ are approximated by \cref{eq:taylorwrtT}, the following can be written
\begin{equation}
  \label{eq:rdMCapprox}
  P_M^{(i+1)} \approxeq P_M^{(i)} + r_d(\mu_R)\cdot d_{T_i},
\end{equation}
where $d_{T_i} = (T_{i+1} - T_{i})$.
While the formulation \textcolor{blue}{\cref{eq:sensitivityMCapprox}} should in principle give a better estimation, it also requires to compute an integral for each evaluation, i.e. \cref{eq:sensH3} for every CP at a different $T_i$, while \cref{eq:rdMCapprox} requires to compute just once a uni-variate integral, the risk density \cref{eq:riskmeasure}.

\subsection*{Numerical Validation Setup}

To explore how the proposed approximation \cref{eq:rdMCapprox} performs, we compute its value and the result of a MC simulation (considered as ground truth) in a slightly modified version of the setup introduced in \Cref{sec:numval}. We maintain the choice of paths and the range of $\sigma$, albeit lowering the resolution, while we impose the combined disk radius $r$ to vary logarithmicly in the range $[10^{-2},1]$: accordingly the value of $d_{T_i}$ at each index will be $r_{i+1} - r_i$.

\subsection*{Discussion of Results}
The graph \Cref{fig:H3approx_weird} shows a comparison of the result of \cref{eq:rdMCapprox} against the ground truth, while the tables \cref{tab:averageabsoluteerror} \cref{tab:relativeabsoluteerror} display both the absolute error and the relative absolute error between the two, i.e. $e_a = | P_M^{(i+1)} - P_M^{(i)} + r_d(\mu_R)\cdot d_{T_i}|$ and $e_r = e_a/P_M^{(i+1)}$.

We see how the proposed approximation follows the ground truth generally well with few exceptions, as both the mean errors $e_a$ and $e_r$, are relatively low, making the approach attractive even in comparison to \cref{eq:sensitivityMCapprox}: \Cref{tab:averageabsoluteerror} show lower error values for \cref{eq:sensitivityMCapprox} in all but few cases, but the differences are not drastic and there is additional computational cost required. The limitations of the approach show up in the case in which the trajectory is far from the obstacle relative to the covariance scaling $\sigma$, e.g. in the set of pairs $\{(\mu_B,10^{-3}),(\mu_C,10^{-3}), (\mu_C,10^{-2})\}$, but while the relative error $e_r$ is high, the absolute one is limited. The reason lies in the fact that the risk density \cref{eq:riskmeasure} in these cases evaluate very close to $0$ and \cref{eq:rdMCapprox} output is very close to $P_M^{(i)}$, shown as an apparent shift of the graphs in \cref{fig:H3approx_weird}.

\begin{figure}
  \centering
  \includegraphics[clip,width=0.9\linewidth]{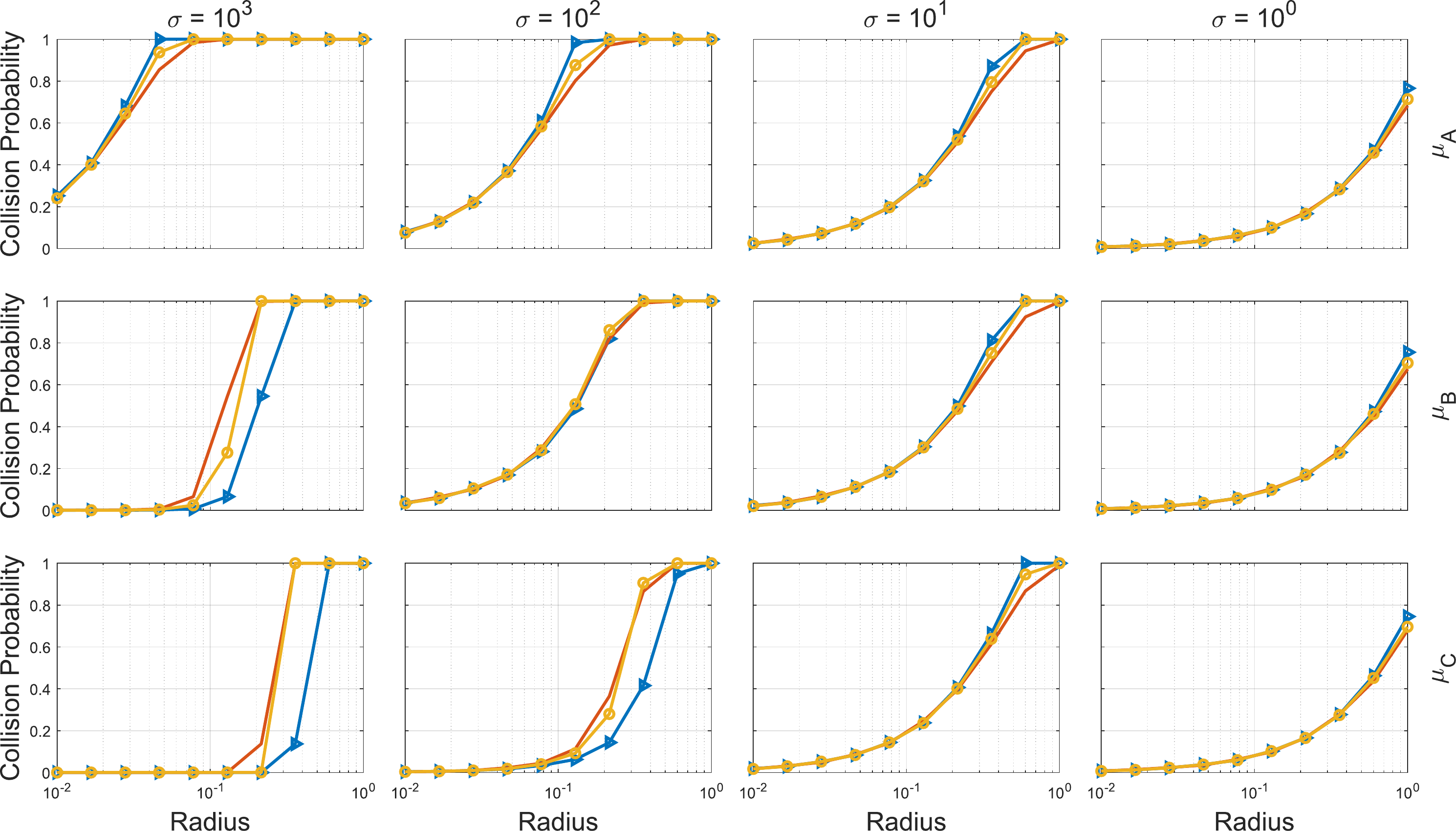}
  \caption{Comparison of MC ground truth against approximations. The graphs show the CP estimate against the radius' dimension. The MC ground truth is showed in Red, while the approximations are depicted in Blue \cref{eq:rdMCapprox} and Yellow \cref{eq:sensitivityMCapprox}. The grid is disposed to make the scrutinized pairing of path and $\sigma$ explicit. \label{fig:H3approx_weird}}
\end{figure}

\begin{table}[]
  \centering
  {\scriptsize
    \begin{tabular}{|l|r|r|r|r|}
      \hline
      $\bm{\mu} \ \backslash \ \bm{\sigma}$ & $\bm{10^{-3}}$                                                  & $\bm{10^{-2}}$                                                     & $\bm{10^{-1}}$                                                      & $\bm{1}$                                                       \\ \hline
      $\bm{\mu_A}$                          & \textcolor{RoyalBlue}{2.33} \textcolor{BurntOrange}{1.25} [\%]  & \textcolor{RoyalBlue}{2.63} \textcolor{BurntOrange}{1.27} [\%]     & \textcolor{RoyalBlue}{2.23} \textcolor{BurntOrange}{1.26} [\%]      & \textcolor{RoyalBlue}{1.23} \textcolor{BurntOrange}{0.57} [\%] \\
      $\bm{\mu_B}$                          & \textcolor{RoyalBlue}{9.96} \textcolor{BurntOrange}{3.15} [\%]  & \textcolor{RoyalBlue}{0.54} \textcolor{BurntOrange}{0.75}  [\%]    & \textcolor{RoyalBlue}{2.21} \textcolor{BurntOrange}{1.43}  [\%]     & \textcolor{RoyalBlue}{1.33} \textcolor{BurntOrange}{0.72} [\%] \\
      $\bm{\mu_C}$                          & \textcolor{RoyalBlue}{10.00} \textcolor{BurntOrange}{1.36} [\%] & \textcolor{RoyalBlue}{7.88} \textcolor{BurntOrange}{1.56}     [\%] & \textcolor{RoyalBlue}{2.22} \textcolor{BurntOrange}{1.39}      [\%] & \textcolor{RoyalBlue}{1.06} \textcolor{BurntOrange}{0.49} [\%] \\ \hline
    \end{tabular}
  }
  \caption{Mean absolute error w.r.t. the radius for \cref{eq:rdMCapprox} (\textcolor{RoyalBlue}{Blue}) and \cref{eq:sensitivityMCapprox} (\textcolor{BurntOrange}{Yellow}). \label{tab:averageabsoluteerror}}
\end{table}

\begin{table}[]
  \centering
  {\scriptsize
    \begin{tabular}{|l|r|r|r|r|}
      \hline
      $\bm{\mu} \ \backslash \ \bm{\sigma}$ & $\bm{10^{-3}}$                                                   & $\bm{10^{-2}}$                                                      & $\bm{10^{-1}}$                                                      & $\bm{1}$                                                       \\ \hline
      $\bm{\mu_A}$                          & \textcolor{RoyalBlue}{3.09} \textcolor{BurntOrange}{1.56} [\%]   & \textcolor{RoyalBlue}{8.81} \textcolor{BurntOrange}{1.95} [\%]      & \textcolor{RoyalBlue}{4.15} \textcolor{BurntOrange}{2.64} [\%]      & \textcolor{RoyalBlue}{4.78} \textcolor{BurntOrange}{3.69} [\%] \\
      $\bm{\mu_B}$                          & \textcolor{RoyalBlue}{45.97} \textcolor{BurntOrange}{32.51} [\%] & \textcolor{RoyalBlue}{2.26} \textcolor{BurntOrange}{2.37}  [\%]     & \textcolor{RoyalBlue}{4.36} \textcolor{BurntOrange}{3.08}  [\%]     & \textcolor{RoyalBlue}{4.23} \textcolor{BurntOrange}{3.26} [\%] \\
      $\bm{\mu_C}$                          & \textcolor{RoyalBlue}{18.63} \textcolor{BurntOrange}{9.95} [\%]  & \textcolor{RoyalBlue}{22.56} \textcolor{BurntOrange}{8.78}     [\%] & \textcolor{RoyalBlue}{4.68} \textcolor{BurntOrange}{3.56}      [\%] & \textcolor{RoyalBlue}{5.01} \textcolor{BurntOrange}{4.14} [\%] \\ \hline
    \end{tabular}
  }
  \caption{Mean relative absolute error w.r.t. the radius for \cref{eq:rdMCapprox} (\textcolor{RoyalBlue}{Blue}) and \cref{eq:sensitivityMCapprox} (\textcolor{BurntOrange}{Yellow}).\label{tab:relativeabsoluteerror}}
\end{table}

%% file: RDopt.tex
To better explore the applicability of the proposed approach, in this section we show how \cref{eq:taylorwrtT} can be used in an open-loop trajectory optimization setting.

As already explored in \cref{sec:approach}, consider the set of obstacles $\mathbb{O}$, each $O_i$ with its associated radius $T_i$ and distribution $\mathcal{N}_i$. Denote as $\mathbb{T}$ the vector of the obstacles' radii and $rd_{\mathbb{O}}$ the vector of all the risk densities \cref{eq:riskmeasure} computed for each obstacle. The probability estimation is then given by
\begin{equation}
  \label{eq:riskapproxmult}
  P_a = rd_{\mathbb{O}} \cdot \mathbb{T}.
\end{equation}
We give two examples of how this formulation can be used in an optimization setting.

\subsection*{Setup}

A robot is to plan a trajectory $\gamma$ to reach a point hidden behind two \enquote{L} shaped corners, while keeping its probability of collision under a certain value.

The trajectory is parametrized as a Bezier curve of the $4$th order $\gamma(s,x_c): \mathbb{S}\times \mathbb{R}^{2\times5}\to \mathbb{R}^2$, of which we optimize the control points $x_c$ to minimize the length of the curve. The obstacles are decomposed in small disks, $r=0.1$, each distributed normally around their nominal position with identical covariance $\Sigma_i = \sigma \cdot I$. The same setting is investigated under different $\sigma$ to analyze the effect of varying uncertainty on the optimization.

The optimization problem takes then the following form.
\begin{align}
  \label{eq:cornerconstrained}
  \min_{x_c} & \int_{0}^1 \left| \frac{d \gamma}{ds} \right| ds &           \\
  s.t.       &                                                  & \nonumber \\
             & P_a(x_c) \leq c_{max}                            & \nonumber \\
             & \gamma(0,x_c) = x_0                              & \nonumber \\
             & \gamma(1,x_c) = x_F                              & \nonumber
\end{align}
The constrained problem is implemented as unconstrained through the use of a logarithmic barrier function \cite{nocedalNumericalOptimization2006a} for the first constraint, i.e. an additional term $-\log(c_max-P_a(x_c))$ is added to the objective function. In this case we have set $c_{max} = 0.2$.
This formulation has been consequently discretized, with $N_d = 300$, and optimized using the \texttt{fminunc} function in MATLAB. Each optimization iteration, one for each value of $\sigma$, is initialized by a fixed feasible configuration of $x_c$.
To compare the efficacy of this method, both a Monte Carlo simulation with $N_M = 10^4$ trials, and \cref{eq:stagewise}\footnote{The summation index set in this case is the cartesian product of the obstacle indexes and the discretization steps.} are computed on the optimized paths.

\subsection*{Results}

\begin{figure*}
  \centering
  \includegraphics[clip,width=0.9\textwidth]{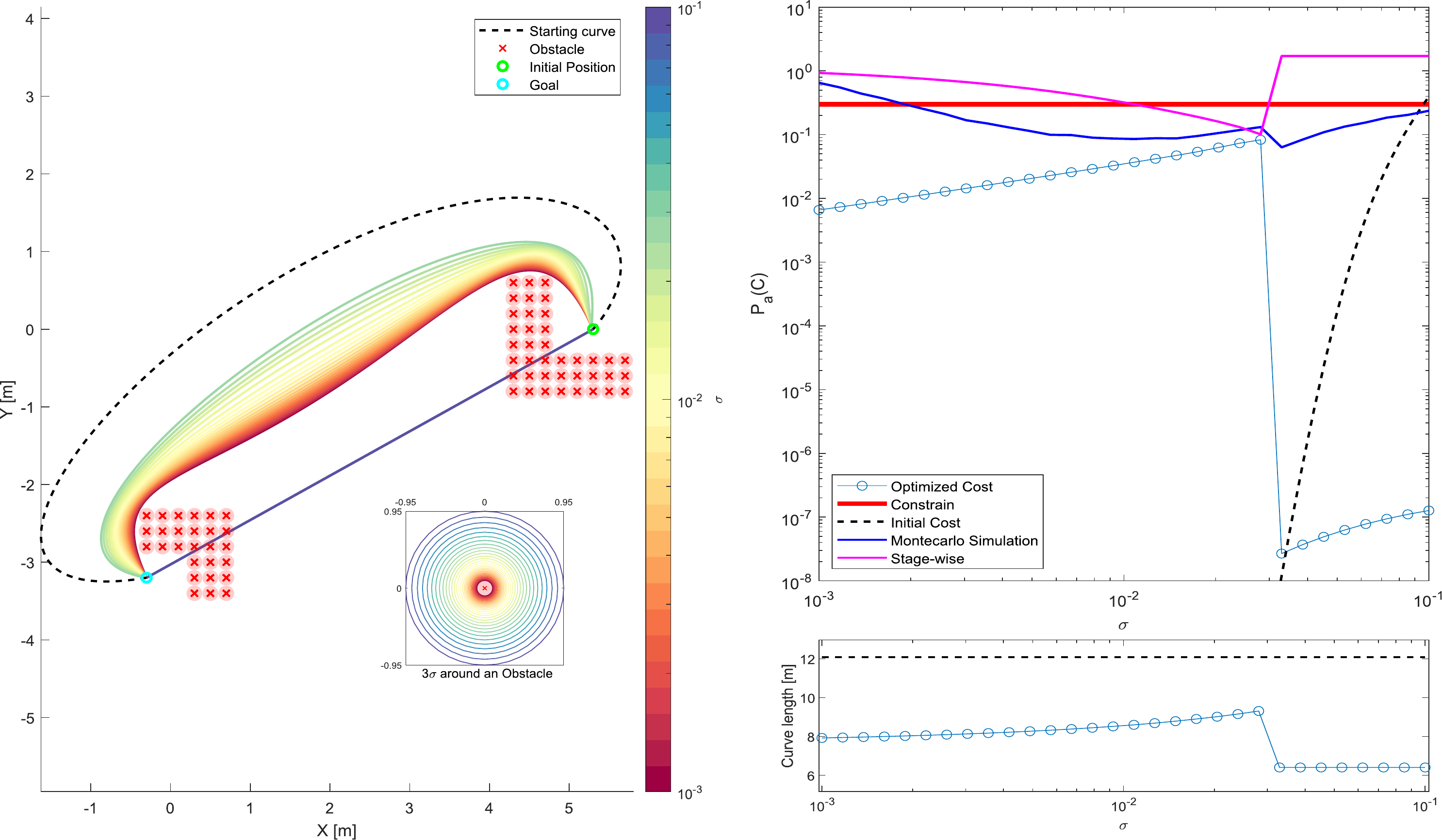}
  \caption{Result of the constrained optimization. Left: Trajectories follow a color gradient highlighting the value of the obstacles' $\sigma$ considered. The square in the bottom right shows how large a circle with radius $3$-std is. The obstacles are shown in red in their nominal position. Right: Plots detailing the length of the trajectory and the estimated probability for each $\sigma$ iteration. The probability estimated by a Monte Carlo simulation and by \cref{eq:stagewise} are also shown. \label{fig:cornercostrained}}
\end{figure*}

\begin{table*}[]
  \centering
  {\scriptsize
    \begin{tabular}{|l|r|r|r|r|r|r|r|r|r|r||r|}
      \hline
      $\bm{\sigma}$           & $\bm{10^{-3}}$            & $\bm{10^{-\nicefrac{25}{9}}}$ & $\bm{10^{-\nicefrac{23}{9}}}$ & $\bm{10^{-\nicefrac{21}{9}}}$ & $\bm{10^{-\nicefrac{19}{9}}}$ & $\bm{10^{-\nicefrac{17}{9}}}$ & $\bm{10^{-\nicefrac{15}{9}}}$ & $\bm{10^{-\nicefrac{13}{9}}}$ & $\bm{10^{-\nicefrac{11}{9}}}$     & $\bm{10^{-1}}$              & $\mathbb{E}[|E_R|]$ \\ \hline
      $E_R [\%]$ \textbf{RD}  & $\color{red} \bm{-64.45}$ & $\color{red} -36.39$          & $-19.75$                      & $-11.24$                      & $-7.28$                       & $-4.85$                       & $-4.14$                       & $-4.86$                       & $-10.91$                          & $-18.51$                    & $16.54$             \\
      $E_R [\%]$  \textbf{SW} & $27.99$                   & $45.78$                       & $\color{BurntOrange}50.10$    & $\color{BurntOrange}44.50$    & $\color{BurntOrange}33.46$    & $21.41$                       & $9.49$                        & $-3.03$                       & $\color{BurntOrange}\bm{160.37}$  & $\color{BurntOrange}152.77$ & $64.48$             \\ \hline \hline
      $\bm{\sigma}$           & $\bm{10^{-3}}$            & $\bm{10^{-\nicefrac{25}{9}}}$ & $\bm{10^{-\nicefrac{23}{9}}}$ & $\bm{10^{-\nicefrac{21}{9}}}$ & $\bm{10^{-\nicefrac{19}{9}}}$ & $\bm{10^{-\nicefrac{17}{9}}}$ & $\bm{10^{-\nicefrac{15}{9}}}$ & $\bm{10^{-\nicefrac{13}{9}}}$ & $\bm{10^{-\nicefrac{11}{9}}}$     & $\bm{10^{-1}}$              & $\mathbb{E}[|E_P|]$ \\ \hline
      $E_P [\%]$  \textbf{RD} & $\color{red} -98.99$      & $\color{red} -97.55$          & $-93.91$                      & $-86.12$                      & $-73.85$                      & $-56.74$                      & $-43.41$                      & $-36.91$                      & $\bm{-100.00}$                    & $\bm{-100.00}$              & $80.78$             \\
      $E_P [\%]$  \textbf{SW} & $42.99$                   & $122.74$                      & $\color{BurntOrange}238.23$   & $\color{BurntOrange}340.98$   & $\color{BurntOrange}339.34$   & $250.42$                      & $99.55$                       & $-23.02$                      & $\color{BurntOrange}\bm{1469.93}$ & $\color{BurntOrange}825.33$ & $491.77$            \\ \hline
    \end{tabular}
  }
  \caption{Error and Relative Error w.r.t. Monte Carlo simulation for \cref{eq:riskapproxmult} (\textbf{RD}) and \cref{eq:stagewise} (\textbf{SW}). Negative numbers indicate an under estimation of probability. Values in red indicates an actual infeasibility of the trajectory while those in orange indicates feasible trajectories that would have been mischaracterized. Means are computed considering also values not shown in the tables. The maxima are boldened. \label{tab:errorobsta}}
\end{table*}

From \Cref{fig:cornercostrained} we see how the estimate given by \cref{eq:riskapproxmult} is always underestimating the actual probability, while \cref{eq:stagewise} overestimates the collision probability in all but one case.
The table \cref{tab:errorobsta} shows the error $E_P = P_{est}-P_M$ and the relative error $E_R = E_P/P_M$  respective to the Monte Carlo simulation. We see how the mean of $|E_P|$ and $|E_R|$ are both significantly lower in \cref{eq:riskapproxmult} case. In the individual scenarios our method produces a worse estimate in the absolute value only in the first cases, $\sim\times2.3$ worse, while doing $\sim \times 1.24 - 26 $ better in all other cases. Each iteration was optimized to tolerance in $\sim 0.3s$.
It is to note however that our method finds unfeasible solutions, highlighted in red in \cref{tab:errorobsta}, while \cref{eq:stagewise} mischaracterize potentially feasible paths, in orange.

\Cref{fig:cornercostrained} shows results consistent with the analysis in \Cref{sec:discussion}. While the optimization finds trajectories that are very close to the nominal position of the obstacles while $\sigma$ is low, intermediate values of the parameter lead to trajectories that are \enquote{further away} from the obstacles and higher values result in straight paths. We suppose that this is due to the same peaking phenomenon we have discussed about the probability estimates in \Cref{sec:discussion}.
When the uncertainty is sufficiently low, the optimization is lead to avoid the obstacles. However, when the uncertainty parameter $\sigma$ is high enough, the obstacles distributions widen to resemble a uniform distribution.
Consequently, the probability of collision $P_a$ when circumventing obstacles becomes comparable to that of taking a straight path, which is then preferred. This phenomenon is illustrated in the second plot from the right in Figure \ref{fig:cornercostrained}.
When the change of behaviour happens, the probability estimate of the optimized path, in light blue, is roughly equal to the one, in black, of the starting path: the prominent part of $P_a$ is given by the initial and final position and moving $x_c$ in any direction does not lower $P_a$ significantly.

\begin{remark}
  As highlighted by \Cref{fig:cornercostrained}, our method correctly captures the change of behaviour in probability displayed at $\sigma \simeq 0.033$: both the Monte Carlo estimate and \cref{eq:riskapproxmult} trend downwards, while \cref{eq:stagewise} captures the opposite behaviour.
\end{remark}

%% file: RDconclusions.tex
This paper presents an analysis of the sensitivity of collision probability to small perturbations in the robot/object dimensions, and introduces the idea of \enquote{risk density}.  We show that this concept carries over different assumptions about the interplay of collision events, modeling different practical situations and offering a theoretical link between the hypotheses.

The derivation of such approximation starts from analyzing the intersection probability of two uncertain shapes and then generalizing it to a continuous path.
After discussing this theoretical property of risk density, we showed how a probability approximation scheme based on it can be applied to the case of a robot moving in a deterministic unknown environment with uncertain initial conditions.
We have thus shown how the proposed approximation offers accuracy and performance comparable to or better than other popular probability estimation methods. 
We have also displayed how the risk density can be used to avoid computing a Montecarlo simulation when a slight change in environment occurs.
The proposed risk-density based approximation does not introduce the need for tuning parameters that are not intrinsic to the problem.
Finally we have shown the usefulness of the method in an optimization setting.
\subsection*{Limitations and Future Research}

Our numerical comparisons were based on assumption (H3) only, i.e. uncertainty only in the initial condition. While valid, the assumption is a simplification of the true stochastic nature of the dynamics. Describing the collision event as a stopped process eludes considering the conditioning between states: a better approach would consider that.
Our approach to computing collision probability does not currently consider non-Gaussian and non-stationary distributions for the uncertainty. Therefore, further investigation in this regard is warranted.
The fact that the \emph{Risk Density} definition coincides between (H2) and (H3) encourages to explore its application to the (H2) case as well. 
In the context of \emph{Risk Density}, exploring further the theoretical link between (H2) and (H3) and how it can be translated to practical approximations also under other hypotheses is of interest.
Moreover, while in this manuscript we have carried out the theoretical derivation of the same method in $\mathbb{R}^3$, its efficacy and the further generalizations are to be explored.

Finally the method explored does not factor rotations of the objects in the probability computation. This deficiency is exposed as soon as we do not introduce a overbounding ball and consider directly the robot shape. This issue could be ameliorated by the use of specialized distributions (e.g. Von-Mises or Wrapped Gaussian), or by working directly in a differential geometry setting. We plan to explore these ideas in the future.

%% file: RDack.tex
This work was supported by the EU Project DARKO - Dynamic and Agile production Robots that learn and optimize Knowledge and Operations (Grant ID 101017274), Fit4MedRob - Fit for Medical Robotics Foundation (PNC0000007).
Moreover, this work was carried out within the framework of the project \enquote{RAISE - Robotics and AI for Socio-economic Empowerment} and has been supported by European Union - NextGenerationEU.
The views and opinions expressed are, however, those of the authors alone, and do not necessarily reflect those of the European Union or the European Commission. Neither the European Union nor the European Commission can be held responsible for them.

%% file: RDappendix.tex
\subsection{Risk Density in $\mathbb{R}^3$}
\subsubsection{Probability Approximation}
Consider the sphere $b\subseteq\mathbb{R}^3$ bounding the Minkowksi difference between the robot and the object shape from outside. The tube swept by the translation of $b$ along the path $\mu_R(s)$ can be parametrized by a function $\Phi(s,t,\theta): \mathbb{S} \times [0,\overline{T}] \times [0, 2\pi) \to \mathbb{R}^3$ as
\begin{equation}
    \label{eq:3Dtube}
    \Phi (s,t,\theta) = \mu(s) + t \cdot \left( - \hat{N}(s)\cos(\theta) + \hat{B}(s)\sin(\theta) \right).
\end{equation}
Here $\hat{T}(s) = \frac{d\mu_R(s)}{ds}/|\frac{d\mu_R(s)}{ds}|$ is the vector tangent to the curve, $\hat{N}(s) = \frac{d\hat{T}(s)}{ds}/|\frac{d\hat{T}(s)}{ds}|$ is the normal vector defining one of the base vectors of the plane normal to $\mu_R(s)$, the other given by the binormal vector $\hat{B}(s) = \hat{T}(s) \times \hat{N}(s)$.

Computing the formulas corresponding to \cref{eq:probabilityRVparam} and \cref{eq:volterraprob} is straightforward and follows the same procedure as in \Cref{sec:statement}, since \cref{def:H2} and \cref{def:H3} do not depend on the dimension of the space which the sets $D_{RO}(s)$ live in.
The resulting expressions are
\begin{align}
    \label{eq:H3probability3D}
     & P^{a}_{H3}(C) = \nonumber                                                                                                     \\
     & \int_{0}^{2\pi} \int_{0}^{\overline{T}} \int_{0}^{1}\mathcal{N}(\gamma,t,\theta) \Gamma(\gamma,t,\theta)  d\gamma dt d\theta,
\end{align}
\small
\begin{align}
    \label{eq:H2probability3D}
     & P^{a}_{H2}(C)= \nonumber                                                                                                                            \\
     & 1-\exp\left( - \int_0^1 \int_{0}^{2\pi} \int_{0}^{\overline{T}}  \mathcal{N}(\gamma,t,\theta) \Gamma(\gamma,t,\theta)  dt d\theta d\gamma  \right),
\end{align}
\normalsize
where $\mathcal{N}(\gamma,t,\theta) := \mathcal{N}(\Phi(\gamma,t,\theta) - \mu_O|0,\Sigma_T)$, $\Gamma(\gamma,t,\theta) =|\det (\nabla \Phi(\gamma,t,\theta))|$ and the vectors and matrices live in $\mathbb{R}^3$.
\subsubsection{Risk Density approximation in $\mathbb{R}^3$}
Using the Liebnitz integral rule, we write
\begin{align}
    \label{eq:3DsensH3}
     & S_{r}^{H3}(\mu_r,\overline{T}) = \frac{d P^a_{H3}}{dT} = \nonumber                                                     \\
     & \int_{0}^{2\pi} \int_{0}^{1}\mathcal{N}(\gamma,\overline{T},\theta) \Gamma(\gamma,\overline{T},\theta) d\gamma d\theta
\end{align}
and
\footnotesize
\begin{align}
    \label{eq:3DsensH2}
     & S_{r}^{H2}(\mu_r,\overline{T}) = \frac{d P^a_{H2}}{dT} = \nonumber                                                                                                             \\
     & S_{r}^{H3}(\mu_r,\overline{T}) \exp \left(- \int_0^{\overline{T}} \int_{0}^{2\pi} \int_{0}^{1}\mathcal{N}(\gamma,t,\theta) \Gamma(\gamma,t,\theta) d\gamma d\theta dt \right).
\end{align}
\normalsize
Notice how $\Phi(s,0,\theta) = \mu_R(s)$, however $\Gamma(s,0,\theta)=0$ $\forall s,\theta$, since the Jacobian matrix of $\Phi$ is singular when $t=0$.
\begin{equation}
    \label{eq:gamma3D}
    \Gamma(s,t,\theta) =
    \left| \det \left(
    \begin{bmatrix}
            \frac{d\Phi}{ds}
             & \frac{d \Phi}{dt}
             & \frac{d\Phi}{d\theta}
        \end{bmatrix}
    \right) \right|
\end{equation}
where
\begin{align}
    \label{eq:phids}
     & \frac{d \Phi}{ds} = \hat{T} \left|\frac{d \mu_R}{ds}\right| + t \overbrace{\left( \frac{d\hat{B}}{ds} \sin (\theta) - \frac{d\hat{N}}{ds} \cos (\theta) \right)}^{\textcircled{\raisebox{-0.9pt}{1}}}, \\
    \label{eq:phidt}
     & \frac{d \Phi}{dt} = \hat{B}(s)\sin(\theta) - \hat{N}(s)\cos(\theta),                                                                                                                                   \\
    \label{eq:phidtheta}
     & \frac{d\Phi}{d\theta} = t\cdot(\hat{N}(s)\sin(\theta) + \hat{B}(s)\cos(\theta)).
\end{align}
So the sensitivities in $\overline{T}=0$ both vanish
\begin{equation}
    \label{eq:3Dsens=0}
    S_{r}^{H2}(\mu_r,0) = S_{r}^{H3}(\mu_r,0) = 0.
\end{equation}
Nevertheless, consider the second order sensitivities
\begin{equation}
    \label{eq:3Dsecondorder}
    S_{r'}^{H2/3}(\mu_r,\overline{T}) = \frac{d^2P^a_{H2/3}(c)}{dT^2}.
\end{equation}
By rewriting the determinant inside \cref{eq:gamma3D} as a scalar triple product
\begin{equation}
    \label{eq:tripleproduct}
    \det(\nabla \Phi) = \frac{d\Phi}{ds} \cdot \left( \frac{d\Phi}{dt} \times \frac{d\Phi}{d\theta} \right).
\end{equation}
The vector product part of this expression simplifies as
\begin{align}
    \label{eq:vectorproduct}
     & \hat{B}\sin(\theta) \times  \hat{N} \sin(\theta) t - \hat{N}\cos(\theta) \times \hat{N}\sin(\theta)t  \nonumber     \\
     & + \hat{B}\sin(\theta) \times \hat{B} \cos(\theta) t - \hat{N} \cos(\theta) \times \hat{B} \cos (\theta) t \nonumber \\
     & = \hat{B}\times \hat{N} \sin^2(\theta) t +  \hat{B}\times \hat{N} \cos^2(\theta) t \nonumber                        \\
     & = - t \hat{T},
\end{align}
using the anticommutative propriety of the vector product.
Moreover, using the Frenet-Serret formulas, $\textcircled{\raisebox{-0.9pt}{1}}$ becomes
\begin{equation}
    \label{eq:onerewritten}
    \textcircled{\raisebox{-0.9pt}{1}} = \left( -\tau(s) \hat{N} \sin(\theta) - (\tau(s) \hat{B} - \kappa(s) \hat{T})\cos(\theta) \right),
\end{equation}
where $\kappa: \mathbb{S} \to \mathbb{R}$ and $\tau: \mathbb{S} \to \mathbb{R}$ are the curvature and the torsion of $\mu_R$ respectively\footnote{These have an algebraic form given that $\mu_R$ is of class $\mathcal{C}^3$, but their value is of no interest here.}.
Being the scalar product bilinear and using again \cref{eq:tripleproduct}, \cref{eq:gamma3D} is then
\begin{align}
    \label{eq:twopieces}
    \Gamma(s,t,\theta) = & \left|\hat{T} \left| \frac{d\mu_R}{ds} \right| \cdot (- t \hat{T}) + t \textcircled{\raisebox{-0.9pt}{1}} \cdot (- t \hat{T})\right| \nonumber \\
                         & \left| \left(t \left| \frac{d\mu_R}{ds} \right| + t^2 \kappa(s) \cos(\theta) \right) \right |.
\end{align}
We can then write
\footnotesize
\begin{align}
    \label{eq:3DsecondorderH3}
     & S_{r'}^{H3} = \nonumber                                                                                                                                                                                                                   \\
     & \int_{0}^{2\pi} \int_{0}^{1}\left(\frac{d\mathcal{N}}{dT}(\gamma,\overline{T},\theta) \Gamma(\gamma,\overline{T},\theta) + \mathcal{N}(\gamma,\overline{T},\theta) \frac{d\Gamma}{dT}(\gamma,\overline{T},\theta) \right)d\gamma d\theta.
\end{align}
\normalsize
When \cref{eq:3DsecondorderH3} is evaluated in $\overline{T}=0$ it results in
\begin{align}
    \label{eq:3DsecondorderH3_0}
     & S_{r'}^{H3}(\mu_R,0) = \nonumber                                                                                                             \\
     & \int_{0}^{2\pi} \int_{0}^{1}\mathcal{N}(\gamma,0,\theta) \frac{d\Gamma}{dT}(\gamma,0,\theta) d\gamma d\theta =\text{\footnotemark} \nonumber \\
     & \int_{0}^{2\pi} \int_{0}^{1}\mathcal{N}(\gamma,0,\theta) \left| \frac{d\mu_R}{d\gamma} \right| d\gamma d\theta = \nonumber                   \\
     & 2\pi \int_{0}^{1}\mathcal{N}(\mu_R(\gamma) - \mu_O | 0, \Sigma_T) \left| \frac{d\mu_R}{d\gamma} \right| d\gamma = \nonumber                  \\
     & \pi \cdot rd(\mu_R(\cdot)).
\end{align}
\footnotetext{Here $\Gamma(\gamma,T,\theta) = |f(\gamma,T,\theta)|$. The derivative $\frac{d\Gamma}{dT} = \frac{f}{\Gamma}(|\frac{d\mu_R}{d\gamma}| + T \kappa(\gamma) \cos(\theta))$ simplifies to $|\frac{d\mu_R}{d\gamma}|$ since $\lim_{T\to 0} \frac{f}{\Gamma} = 1$, being $T\geq 0$ by definition, cancelling out the fraction.}
Using \cref{eq:3DsecondorderH3}, we derive
\footnotesize
\begin{align}
    \label{eq:3DsecondorderH2}
     & S_{r'}^{H2} = \frac{d}{dT}(\ref{eq:3DsensH2}) = \nonumber                                                                                                                       \\
     & (S_{r'}^{H3} -  (S_{r}^{H3})^2 ) \exp\left(- \int_0^{\overline{T}} \int_{0}^{2\pi} \int_{0}^{1}\mathcal{N}(\gamma,t,\theta) \Gamma(\gamma,t,\theta) d\gamma d\theta dt \right).
\end{align}
\normalsize
Finally, in $\overline{T}=0$ this is
\begin{equation}
    \label{eq:3DsecondorderH23_0}
    S_{r'}^{H2}(\mu_R,0) = S_{r'}^{H3}(\mu_R,0) = \pi \cdot rd(\mu_R(\cdot)).
\end{equation}
Moreover, an approximation similar in nature to \cref{eq:taylorwrtT} is recovered using a second order Taylor approximation
\begin{equation}
    \label{eq:taylor_tubo}
    P_a(C,T) = rd(\mu_R(\cdot)) \cdot \pi T^2,
\end{equation}
which makes intuitive sense as a \enquote{cylindrical approximation}.

\subsection{Algorithms}
We enumerate here the computational procedures used to approximate probabilities in \Cref{sec:numval} and \Cref{sec:trajopt}.

\begin{algorithm}
    \caption{Monte Carlo Simulation - H1}
    \label{alg:groundtruthH1}
    \begin{algorithmic}
        \State \textbf{Input:} Path function $\mathbf{\mu_R(s)}$ \Comment{Domain: $[0,1]$}
        \State Obstacle mean position $\mathbf{\mu_O}$
        \State Combined Covariance $\mathbf{\Sigma_T}$
        \State Combined Radius $\mathbf{r}$
        \State Number of Monte Carlo Trials $\mathbf{N_t}$
        \State Discretization of the Path $\mathbf{N_s}$
        \State \textbf{Output:} Probability estimate $\mathbf{P_T}$
        \State Initialize collision counter $\mathbf{c=0}$
        \For{$j$ = $1$ to $N_t$}    \Comment{Done in parallel}
        \For{$s_i$ = $0$ to $N_s$}
        \State Sample $\mu^{*ij}_O$ from $\mathcal{N}(\mu_O,\Sigma_T)$
        \State Compute $d^{j}_i=\mu_R(\frac{s_i}{N_s})-\mu^{*ij}_O$
        \If{$norm(d^{j}_i)\leq r$}
        \State $c = c+1$
        \State \textbf{break for}
        \EndIf
        \EndFor
        \EndFor
        \State $\mathbf{\hat{P}= \frac{c}{N_t}}$
    \end{algorithmic}
\end{algorithm}
\begin{algorithm}
    \caption{Monte Carlo Simulation - H2}
    \label{alg:groundtruthH2}
    \begin{algorithmic}
        \State \textbf{Input:} $\mathbf{\mu_R(s)}$, $\mathbf{\mu_O}$, $\mathbf{\Sigma_T}$, $\mathbf{r}$, $\mathbf{N_t}$, $\mathbf{N_s}$
        \State \textbf{Output:} Probability estimate $\mathbf{P_T}$
        \State Initialize collision counter $\mathbf{c=0}$
        \For{$j$ = $1$ to $N_t$}    \Comment{Done in parallel}
        \For{$s_i$ = $0$ to $N_s$}
        \State Sample $\mu^{*ij}_O$ from $\mathcal{N}(\mu_O,\Sigma_T)$
        \If{$s_i>0$} \Comment{Na\"ive Rejection Sampling}
        \State Compute $dp^{j}_i=\mu_R(\frac{s_i-1}{N_s})-\mu^{*ij}_O$
        \While{$norm(dp^{j}_i)\leq r$}
        \State Sample $\mu^{*ij}_O$ from $\mathcal{N}(\mu_O,\Sigma_T)$
        \State Compute $dp^{j}_i=\mu_R(\frac{s_i-1}{N_s})-\mu^{*ij}_O$
        \EndWhile
        \EndIf
        \State Compute $d^{j}_i=\mu_R(\frac{s_i}{N_s})-\mu^{*ij}_O$

        \If{$norm(d^{j}_i)\leq r$}
        \State $c = c+1$
        \State \textbf{break for}
        \EndIf
        \EndFor
        \EndFor
        \State $\mathbf{\hat{P}= \frac{c}{N_t}}$
    \end{algorithmic}
\end{algorithm}
\begin{algorithm}
    \caption{Monte Carlo Simulation - H3}
    \label{alg:groundtruth}
    \begin{algorithmic}
        \State \textbf{Input:} $\mathbf{\mu_R(s)}$, $\mathbf{\mu_O}$, $\mathbf{\Sigma_T}$, $\mathbf{r}$, $\mathbf{N_t}$, $\mathbf{N_s}$
        \State \textbf{Output:} Probability estimate $\mathbf{P_T}$
        \State Initialize collision counter $\mathbf{c=0}$
        \For{$j$ = $1$ to $N_t$}    \Comment{Done in parallel}
        \State Sample $\mu^{*j}_O$ from $\mathcal{N}(\mu_O,\Sigma_T)$
        \For{$s_i$ = $0$ to $N_s$}
        \State Compute $d^{j}_i=\mu_R(\frac{s_i}{N_s})-\mu^{*j}_O$
        \If{$norm(d^{j}_i)\leq r$}
        \State $c = c+1$
        \State \textbf{break for}
        \EndIf
        \EndFor
        \EndFor
        \State $\mathbf{\hat{P}= \frac{c}{N_t}}$
    \end{algorithmic}
\end{algorithm}
\begin{algorithm}
    \caption{Probability Estimate with Risk Density, Single Obstacle}
    \label{alg:rdestimate}
    \begin{algorithmic}
        \State \textbf{Input:} $\mathbf{\mu_R(s)}$, $\mathbf{\mu_O}$, $\mathbf{\Sigma_T}$, $\mathbf{r}$
        \State \textbf{Output:} Probability estimate $\mathbf{P_a}$
        \State Compute $rd$ \cref{eq:riskmeasure} with inputs $\mu_R(s),\mu_O,\Sigma_T$
        \State $\mathbf{P_a = rd \cdot r}$
    \end{algorithmic}
\end{algorithm}
\begin{algorithm}
    \caption{Probability Estimate with Risk Density, Multiple Obstacles}
    \label{alg:rdmultiple}
    \begin{algorithmic}
        \State \textbf{Input:} $\mathbf{\mu_R(s),\overline{\mu}_O,\overline{\Sigma}_T, \overline{r}}$ \Comment{Barred variables are vectors of their respective elements.}
        \State \textbf{Output:} Probability Estimate $\mathbf{P_a}$
        \State $\overline{rd} =$ \cref{eq:multipleRD} with inputs $\mu_R(s),\overline{\mu}_O,\overline{\Sigma}_T$
        \State $\mathbf{P_a = rd_{\mathbb{O}}^{\top} \cdot \overline{r}}$
    \end{algorithmic}
\end{algorithm}
\break